\documentclass{article}

\usepackage{arxiv}

\usepackage[utf8]{inputenc} 
\usepackage[T1]{fontenc}    
\usepackage{hyperref}       
\usepackage{url}            
\usepackage{booktabs}       
\usepackage{amsfonts}       
\usepackage{nicefrac}       
\usepackage{microtype}      
\usepackage{lipsum}
\usepackage{graphicx}
\usepackage{xcolor}

\graphicspath{ {./images/} }

\usepackage{amsthm}
\usepackage{amssymb}

\usepackage{thmtools}
\usepackage{thm-restate}

\newtheorem{theorem}{Theorem}

\newtheorem{corollary}[theorem]{Corollary}
\newtheorem{assumption}{Assumption}

\newcommand{\1}{\mathbf{1}}

\usepackage{amsmath,amsfonts,bm}

\def\eqref#1{(\ref{#1})}

\def\1{\bm{1}}

\DeclareMathAlphabet{\mathsfit}{\encodingdefault}{\sfdefault}{m}{sl}
\SetMathAlphabet{\mathsfit}{bold}{\encodingdefault}{\sfdefault}{bx}{n}

\def\gB{{\mathcal{B}}}
\def\gC{{\mathcal{C}}}

\def\gG{{\mathcal{G}}}
\def\gH{{\mathcal{H}}}

\def\gN{{\mathcal{N}}}

\def\gP{{\mathcal{P}}}
\def\gQ{{\mathcal{Q}}}

\def\gS{{\mathcal{S}}}
\def\gT{{\mathcal{T}}}

\def\gV{{\mathcal{V}}}

\def\gY{{\mathcal{Y}}}

\newcommand{\E}{\mathbb{E}}

\newcommand{\R}{\mathbb{R}}

\DeclareMathOperator*{\argmax}{arg\,max}

\usepackage{algorithm}
\usepackage{algorithmic}

\usepackage{multirow}
\usepackage{makecell}
\usepackage{threeparttable}
\usepackage{siunitx}

\usepackage{caption}
\usepackage{subcaption}

\usepackage{csquotes}

\usepackage[most]{tcolorbox}

\usepackage{fancyvrb}
\usepackage{fvextra}
\newtcolorbox{prompt}[1][]{
floatplacement=ht,       
  float,                    
  enhanced,
  breakable,
  colback=black!5,      
  colframe=black!75,    
  fonttitle=\bfseries,  
  arc=4mm,
top=4mm,
bottom=4mm,
left=5mm,
right=5mm,
  title=#1
}

\title{Mitigating Over-Optimization in PRM-Guided Search in Mathematical Reasoning by Optimizing the Guide}

\author{
Taejong Joo \& Diego Klabjan \\
  Department of Industrial Engineering \& Management Sciences \\
  Northwestern University \\
  \texttt{\{taejong.joo,d.klabjan\}@northwestern.edu} 
}

\date{}

\begin{document}
\maketitle

\begin{abstract}
\noindent Process reward models (PRMs) provide dense step-level guidance for search-based reasoning, enabling inference-time compute to be allocated toward promising partial solutions. 
However, recent evidence suggests that PRM-guided search can over-optimize imperfect process rewards, pruning viable trajectories while expanding spurious ones. 
In this work, we theoretically show that directly leveraging PRM score is vulnerable to verifier noise through an extreme-value effect: non-viable prefixes become more likely to receive spuriously high scores as reasoning depth increase. 
Therefore, we formulate the PRM-guided search as a robust optimization problem over plausible reward perturbations, termed maximin PRM-guided search, leading to a training-free robust process supervision method that preserves promising alternatives when step-level scores are noisy. 
Maximin PRM-guided search mitigates this failure mode by reducing sensitivity to over-optimized PRM outliers. 
Without fine-tuning or online adaptation, maximin search consistently improves the PRM-guided search by 17–35\% on average, outperforming outcome- and step-level baselines in 14 out of 16 settings.
Our source code is available at \url{https://github.com/tjoo512/maximin-search}.
\end{abstract}

\section{Introduction}
Large language models (LLMs) perform strongly on tasks that benefit from explicit intermediate reasoning, such as mathematical problem solving and code generation, and chain-of-thought prompting is a simple and effective way to elicit this behavior \citep{wei2022chain, kojima2022large}. 
Importantly, performance on these tasks depends not only on the base model capability, but also on how much computation is allocated at inference time, including via sampling \citep{wang2022self}, search \citep{yao2023tree, besta2024graph}, or verification \citep{cobbe2021training}.
For example, Tree of Thoughts improves performance by explicitly searching over partial reasoning states and backtracking from unpromising ones rather than following a single linear chain of thoughts \citep{yao2023tree}.
This has motivated treating inference-time compute as a new scaling axis for reasoning models \citep{brown2024large, snell2024scaling}.

For search-based inference to improve reasoning, the system must determine which partial trajectories merit further computation \citep{yao2023tree, besta2024graph}. Process reward models (PRMs) provide a natural mechanism for this \citep{uesato2022solving, lightman2023let}.
Whereas outcome reward models score only final solutions, PRMs assign values to intermediate reasoning steps, yielding a denser supervision and control signal for multi-step search. 
This finer-grained guidance is especially useful when the search space is large and the compute budget is limited, since it enables early pruning of low-value trajectories \citep{feng2023alphazero, chen2024alphamath, wang2024math}. 
Accordingly, PRM-guided search has become an effective way to translate additional inference-time compute into gains in reasoning performance.

This efficiency, however, depends critically on the reliability of the reward estimates guiding search. A PRM is only a learned proxy for latent reasoning quality, and its predictions may degrade due to limited supervision, model misspecification, or distribution shift \citep{lightman2023let, casper2023open}. This issue is especially acute in search-based settings, because the search procedure itself determines which partial trajectories are encountered at inference time \citep{ross2011reduction, zhang2025lessons}. 
For instance, a false negative can permanently prune a promising branch, while a false positive can waste substantial budget on spurious trajectories. 
Consequently, PRM-guided search is vulnerable to \textit{over-optimization} or \textit{reward hacking}, where
optimizing an imperfect objective does not necessarily improve the true objective \citep{coste2023reward, gao2023scaling, skalse2022defining}.

A natural response is to improve the reward model online, for example through test-time training or related adaptation mechanisms \citep{sun2020test, shinn2023reflexion, hosseini2024v}. However, this elevates the computational and algorithmic burden rather than fundamentally resolving the vulnerability. Such approaches introduce additional optimization into the critical path of inference, increasing computational cost and implementation complexity \citep{snell2024scaling}. Furthermore, their stability relies on precisely the low-signal, out-of-distribution regime that already makes reward estimation difficult \citep{gao2023scaling,casper2023open}. 
More importantly, adaptation aims to repair the scorer itself, whereas the core decision problem is how to act when the scorer is inherently uncertain. From this perspective, the central challenge is not solely accurate reward learning, but robust decision-making under reward uncertainty \citep{iyengar2005robust,kumar2020conservative}.

In this work, we address PRM imperfection from that perspective (cf. Figure \ref{fig:illustration}). 
We formulate PRM-guided search as a robust optimization problem in which reward estimates are treated as uncertain rather than exact. 
The key idea is as follows: when reward estimates are unreliable, the search should reflect plausible reward perturbations. 
This formulation leads to inference procedures that are better aligned with the true objective, mitigating over-optimization of imperfect process rewards and yielding strong robustness to noisy reward estimates.
More broadly, our results suggest that scaling inference-time compute in PRM-guided search requires not only more search, but also better algorithmic robustness to the noisy score evaluations to make search effective in the first place.

Our contributions are as follows:
\begin{itemize}
    \item \textbf{A rigorous characterization of PRM over-optimization.}
    We formalize PRM-guided search as optimization with an imperfect step-level verifier and identify a concrete failure mechanism: maximization over noisy process scores can select non-viable prefixes even when viable continuations are present. 
    Our theoretical analysis quantifies how this risk depends on verifier uncertainty, oracle prefix separability, generator coverage, search width, and reasoning depth.
    \item \textbf{A principled robust formulation for PRM-guided search.}
    We recast step-level guidance as robust optimization under plausible verifier perturbations. 
    This yields a decision rule that avoids treating learned PRM scores as exact and instead preserves trajectories that remain competitive under reward uncertainty.
    \item \textbf{A training-free algorithm with strong empirical gains.}
    We instantiate the robust formulation as a simple plug-in procedure for PRM-guided decoding. 
    Without fine-tuning or online adaptation, maximin search consistently improves the PRM-guided search by 17–35\% on average, outperforming outcome- and step-level baselines in 14 out of 16 settings.
\end{itemize}

\begin{figure}
    \centering
    \includegraphics[width=0.98\linewidth]{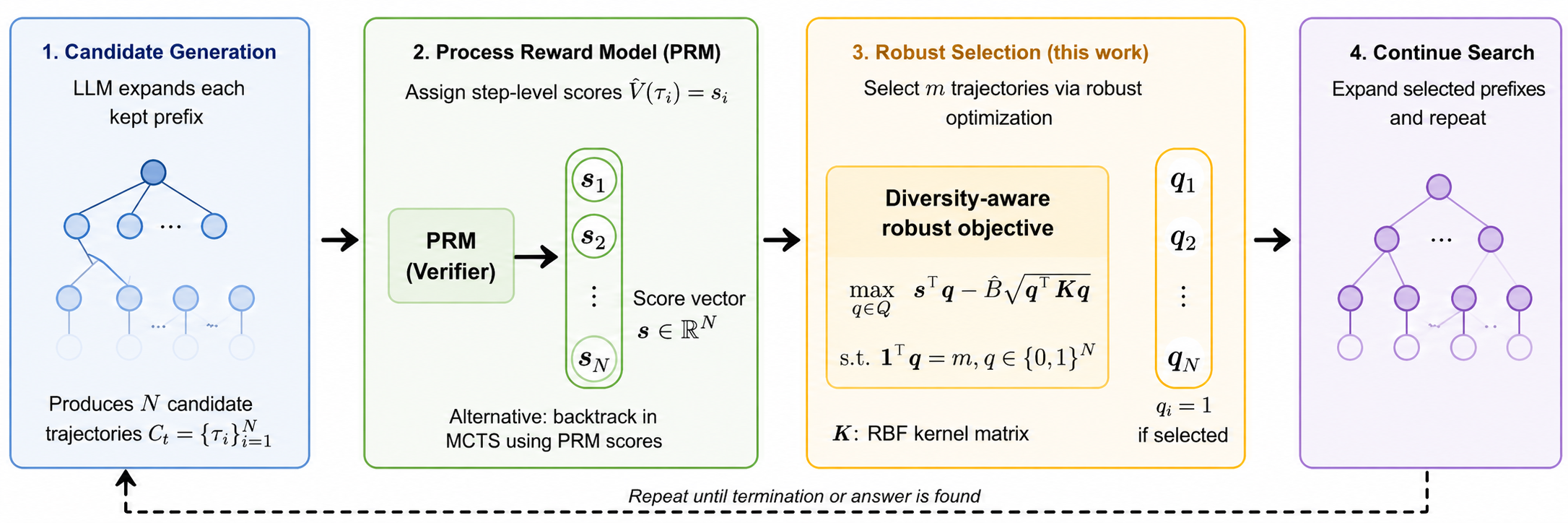}
    \caption{\textbf{Maximin PRM-guided search.} At each pruning step, the LLM first generates candidate reasoning trajectories, then a PRM assigns step-level scores. Instead of greedily keeping only the highest-scoring candidates, our robust selector keeps $m$ prefixes that balance high PRM scores with diversity under an RBF kernel $K$, making the search less sensitive to noisy PRM scores. The selected prefixes are then expanded, and the process repeats.
    }
    \label{fig:illustration}
\end{figure}

\section{Related Work} 
\paragraph{Step-level guidance with PRMs.}
A growing body of work uses step-level evaluation to improve mathematical reasoning at inference time, including self-evaluation-guided beam search \citep{xie2023self}, AlphaZero-style tree search \citep{feng2023alphazero}, step-level beam search \citep{chen2024alphamath}, and actor--critic style decoding \citep{dong2025enhancing}. 
Across these methods, a generator proposes candidate reasoning steps and a learned evaluator scores partial trajectories to guide search, improving performance by allocating compute more selectively across candidate chains.

Our work is most closely related to recent \emph{self-improvement} approaches that address failures from reward misspecification. ReST-MCTS* \citep{zhang2024rest} combines PRM-guided search with iterative updates to the policy and value model; V-STaR \citep{hosseini2024v} similarly improves  output verification through iterative refinement; and PRIME \citep{cui2025process} mitigates reward hacking by updating rewards from rollout outcomes. These methods reduce the mismatch between search and evaluation by adapting the policy, verifier, or reward model to deployment-induced trajectories. However, they rely on additional training, often incur substantial computational overhead, and in some cases require outcome supervision or ground-truth answers \citep{zhang2024rest, hosseini2024v, cui2025process}.

By contrast, we study over-optimization in PRM-guided search from a robust optimization perspective at inference time. Rather than mitigating reward misspecification through retraining or test-time adaptation, we seek search procedures that remain less sensitive to the PRM error under a fixed generator and fixed PRM. To our knowledge, prior work has largely overlooked this robust-inference perspective on PRM-guided mathematical reasoning.

\paragraph{Mathematical characterization of over-optimization.}
Over-optimization under imperfect reward models has been studied primarily in \emph{outcome-level} inference-time optimization, especially Best-of-N (BoN) decoding and related reranking methods. A growing literature shows that while BoN can improve performance initially, increasing optimization pressure against a misspecified reward model eventually induces reward hacking and degrades the true objective \citep{jinnai2025regularized, ichihara2025evaluation, khalaf2025inference}. Recent theory formalizes this rise-then-fall behavior and characterizes the limitations of standard BoN under reward misspecification \citep{khalaf2025inference, huang2025best}. Together, these results establish over-optimization as a fundamental pathology of inference-time optimization in the simpler outcome-level setting.

By contrast, the step-level setting remains much less understood theoretically, despite growing evidence that PRM-guided search can fail in practice. Recent studies show that PRMs can be unreliable evaluators of partial reasoning states, with errors that compound over long horizons and under distribution shift \citep{zhang2025lessons, khalifa2025process, cinquin2025limits}.
\citet{cinquin2025limits} further show that PRM-guided tree search can underperform simpler outcome-level baselines such as BoN due to inaccurate step-value estimates, with the gap increasing as search depth increases and under out-of-distribution shifts. 
These findings suggest that over-optimization may be more severe in step-level guidance, where reward errors shape not only final selection but also the allocation of search compute across intermediate trajectories.

Our work builds on this literature by providing, to the best of our knowledge, the first mathematical treatment of over-optimization in \emph{PRM-guided step-level search}. Whereas prior theory focuses on outcome-level reranking and prior PRM work mainly focuses on empirical failures, we formalize how PRM misspecification interacts with sequential search and pruning decisions.

\section{Background} \label{sec:background}

\paragraph{Reasoning Models}
We consider a reasoning model that generates a structured output for an input problem $x$ whose ground-truth answer is $y$. 
In mathematical reasoning, this output typically consists of a sequence of intermediate reasoning steps followed by a final answer. Throughout the paper, we analyze inference on a single input $x$ and suppress the dependence on $x$ in the notation when it is clear from context.
Let $s \in \mathcal{S} \triangleq \gV^*$ denote a \emph{thought} as the Kleene closure of a set of tokens $\gV$.
Typically, it corresponds to a semantically coherent reasoning segment separated from adjacent segments by a fixed delimiter ``\texttt{\textbackslash n\textbackslash n}'' \citep{lightman2023let,zhang2025lessons,cinquin2025limits}. A \emph{reasoning trajectory} of length $t$ is an ordered sequence of thoughts $\tau^{(t)} = (x, s_1, \dots, s_t),$ with $s_i \in \gS$ for $i \in [t] \triangleq \{ 1,2, \cdots, t\}$.
We let \(T_{\max}\) denote a maximum reasoning horizon and let $\mathsf T = \inf\{t\le T_{\max}: s_t=\texttt{EOS}\}$ denote the random termination time, with the convention \(\mathsf T=T_{\max}\) if no
\texttt{EOS} token is generated earlier. The distribution of \(\mathsf T\) is induced by the generator and the search procedure, and may depend on the search history. For notational simplicity, we pad positions after \(\mathsf T\) with \texttt{EOS}, so that all trajectories can be represented on the common horizon \(T_{\max}\).
When we consider multiple trajectories with the same length, we use $\tau_i^{(t)}$ to denote the $i$-th trajectory of length $t$ for $i \in [N]$ with $N$ being the maximum number of trajectories. 
For a prefix trajectory $\tau^{(t)} = (x, s_1,\dots,s_t)$ and a new thought $s \in \mathcal{S}$, we write $\tau^{(t)} \oplus s \triangleq (x, s_1,\dots,s_t,s)$ as a prefix extension. Also, a trajectory $\tau^{(t)}$ is a prefix of $\tau^{(t')}$, written $\tau^{(t)} \preceq \tau^{(t')}$, if the first $t$ thoughts of $\tau^{(t')}$ equal $\tau^{(t)}$.

The base reasoning model is specified by a thought-level policy $\pi$, which defines a conditional distribution over the next thought given a partial trajectory: $\pi(s \mid \tau^{(t)}) \in \Delta(\mathcal{S}),$
where $\Delta(\mathcal{S})$ denotes the probability simplex over $\mathcal{S}$. This induces a distribution over complete trajectories by autoregressive factorization, $\pi(\tau^{(T)}) = \prod_{t=1}^{T} \pi(s_t \mid \tau^{(t-1)}),$ with $\tau^{(0)} \triangleq x$. When a trajectory is terminated, a deterministic answer extraction map $\mathrm{ans}: \bigcup_{t=0}^{T_{\max}} \mathcal{S}^t \to \mathcal{Y}$ returns the final output $\hat{y} = \mathrm{ans}(\tau)$ in answer space $\mathcal{Y}$.

\paragraph{Step-level guidance}
To guide search over partial trajectories, we consider a process reward model (PRM) $\widehat{V}$ that assigns a scalar value to each intermediate trajectory, $\widehat{V}(\tau^{(t)}) \in \mathbb{R},$ interpreted as an estimate of the future utility of continuing from $\tau^{(t)}$.
We let the oracle value function $V(\tau^{(t)})$ output a utility of $\tau^{(t)}$, such as the logit of the probability of yielding the ground-truth answer \citep{zhang2025lessons} or whether all intermediate thoughts in $\tau^{(t)}$ are correct \citep{uesato2022solving}.
We assume $V$ is bounded as $\underline V_t\le V(\tau^{(t)})\le \overline V_t$ for all $\tau^{(t)}\in\mathcal S^t$.

As a concrete instance, consider step-level beam search which is the focus of our theoretical analysis as in \citet{xie2023self}. At depth $t$, let $\mathcal{B}_t = \{\tau_i^{(t)}\}_{i=1}^{w_t}$ denote the active beam, where $w_t$ is the beam width. 
Each trajectory $\tau_i^{(t)} \in \mathcal{B}_t$ is expanded by sampling or enumerating candidate next thoughts from the base generator $\pi(\cdot \mid \tau_i^{(t)})$. 
Writing $\mathcal{C}_{t+1} = \bigcup_{i=1}^{w_t} \left\{ \tau_i^{(t)} \oplus s : s \sim \pi(\cdot \mid \tau_i^{(t)}) \right\}$ for the resulting set of candidate continuations, the PRM $\hat{V}$ is used to score each candidate prefix $\tau \in \mathcal{C}_{t+1}$.
The size of $\gC_{t+1}$, $|\gC_{t+1}| \le N$, is the total number of trajectories at step $t$ and $b_t = \lfloor N / w_t \rfloor$ is the branching factor at $t$. We note that the procedure reduces to outcome-level selection (Best-of-N) when $b_t = 1$ and reduces to the step-level greedy decoding when $w_t = 1$.
The next beam is then formed by selecting the top-$w_{t+1}$ trajectories under this score (i.e., $\mathcal{B}_{t+1} \subseteq \mathcal{C}_{t+1}$ with $|\mathcal{B}_{t+1}| = w_{t+1}$) and then repeating the procedure until termination (i.e., collecting $N$ answers).
Thus, $\pi$ determines how candidate reasoning steps are proposed, while $\widehat{V}$ determines how computation is allocated across them.

This formulation captures the essential role of step-level guidance: the search procedure repeatedly expands partial trajectories using the generator and uses the PRM to rank, retain, or prune them before they terminate. Beam search provides the simplest instantiation of this principle, but the same decomposition extends to more general search procedures, such as greedy best-first search \citep{cinquin2025limits} and Monte Carlo tree search (MCTS) \citep{feng2023alphazero,chen2024alphamath}, where $\widehat{V}$ can analogously guide node selection, expansion, or backup. 
However, they introduce substantial additional machinery and cost: in LLM reasoning, each expansion requires generating open-ended natural-language continuations rather than applying a cheap discrete transition. Our goal is not to study how to obtain better reward estimates online, but to understand how to search reliably given a fixed PRM whose estimates may be accurate or inaccurate.

\section{Closer look at an over-optimization in PRM-guided search} \label{sec:bad_step_guide}

The goal of PRM-guided search can be understood as maximizing the average score of the selected trajectories. 
Formally, it can be represented as solving the optimization problem of
\begin{equation} \label{eq:goal_prm}
    \gB_t \in 
        \argmax_{\gB'_t \subseteq \gC_t : | \gB'_t | = w_t} \hat{\E}_{\tau \sim \gB'_t} [V(\tau)] \triangleq 
        \argmax_{\gB'_t \subseteq \gC_t : | \gB'_t | = w_t}  \frac{1}{|\gB'_t|} \sum_{\tau \in \gB'_t} V(\tau),
\end{equation}
where we use the notation $\hat{\E}$ to denote the average score in $\gB'_t$.

Ideally, PRM-guided search is expected to filter unpromising trajectories in $\gC_t$ and therefore compensate an imperfect generator quality, attaining $\hat{\E}_{\tau \sim \gB_t} [V(\tau)] \gg \hat{\E}_{\tau \sim \gC_t} [V(\tau)]$. 
However, in practice, using off-the-shelves PRM for the step-level guidance is often subject to the reward over-optimization problem, resulting in inferior performances than an outcome-level decoding method such as best-of-N   \citep{cinquin2025limits, zhang2025lessons, hu2025prm}. 
In this section, we closely examine the failure mode of PRM-guided decoding by quantifying the error in \eqref{eq:goal_prm} under an imperfect verifier.  
Proofs of all formal claims in this paper are provided in Appendix \ref{appx:proof}.

\paragraph{Setup}
We consider the beam-style search over partial solutions as explain in \S \ref{sec:background}.
Let $\gT^\star \subseteq \gS^{T}$ denote the set of successful full trajectories such that $\mathrm{ans}(\tau^\star) = y$ for all $\tau^\star \in \gT^\star$. 
We define the viable-prefix set as $\widetilde{\mathcal{P}}_t^\star = \{ \tau^{(t)} \in \gS^t : \exists \tau^\star \in \gT^\star \text{ such that } \tau^{(t)} \preceq \tau^\star  \}, $ which contains all prefixes of $t$ thoughts that can still reach the correct answer, and the non-viable-prefix set $\widetilde{\mathcal{P}}_t^- \triangleq \gS^{t} \setminus \widetilde{\mathcal{P}}_t^\star$. 
For the candidate set $\gC_t$ with $|\gC_t| \le N$, we let $\gC_t^\star = \gC_t\cap \widetilde{\gP}_t^\star$ and $\gC_t^- = \gC_t\cap \widetilde{\gP}_t^-$. Also, we define $N_t^\star = |\gC_t^\star|$ and $N_t^- = |\gC_t^-|$ be the number of viable and non-viable prefixes, respectively.

Under this setting, we pose the following regularity conditions for the learned PRM $\hat{V}$ as an estimator of the oracle value function $v$.

\begin{assumption}[Verifier regularity]\label{ass:noise}
The learned verifier satisfies $\widehat V(\tau^{(t)}) = V(\tau^{(t)})+\beta_t(\tau^{(t)})+\varepsilon_t(\tau^{(t)})$ for all $t \in [T_{\max}]$, where $\beta_t:\mathcal S^t\to\mathbb R$ is a deterministic bias function satisfying $|\beta_t(\tau^{(t)})|\le B_t \text{ for all }  \tau^{(t)}\in\mathcal S^t .$
Conditional on the search history $\mathcal H_t$ and the realized candidate set $\mathcal C_t$, the noise variables $\{\varepsilon_t(\tau):\tau\in\mathcal C_t\}$ are independent and identically distributed as $\mathcal N(0,\sigma_t^2).$
Further, we assume $B_1 \ge \dots \ge B_{T_{\max}}$ and $\sigma_1 \ge \dots \ge \sigma_{T_{\max}}$.
\end{assumption}

Regularity conditions in Assumption \ref{ass:noise} mathematically formalizes the observations that early-stage evaluations are highly uncertain because the final answer is distant, whereas later steps offer much clearer outcome signals \citep{hu2025prm, cinquin2025limits}.
The bounded bias condition assumes a uniform error bound for the trajectories, which can be attained under a sufficiently strong PRM. Note that this can be replaced by a probabilistic argument, which makes our analyses probabilistic as well, but we focus on the uniform bound for simplicity.
The Gaussian noise assumption is adopted for analytical tractability and to obtain closed-form high-dimensional risk expressions, which is commonly used in high-dimensional interpolation and double-descent analyses \citep{advani2020high,belkin2020two,nakkiran2019more}. 
Finally, we remark that Assumption~\ref{ass:noise} is used only for the Gaussian extreme-value analysis. It is not intended to model all verifier errors. In particular, systematic verifier misspecification may be correlated across semantically similar prefixes; this correlated component is addressed separately in \S~\ref{subsec:kernel_persp}.

\subsection{Over-optimization under PRM-guided search}
We first isolate the pure reward-maximization regime, namely the greedy setting $w_t = 1$ for $t \in [T_{\max}]$. 
As we discuss at the end of \S\ref{subsec:error_accumulation}, enlarging the beam width raises the number of verifier errors that must coincide to derail the search from one to $w_t$, which mitigates but does not remove the failure. 
We therefore develop the mechanism for the over-optimization analysis in the simplest $w_t=1$ case and recover the general beam-search statement as a direct corollary.
To avoid degenerate cases, we assume throughout that $\gC_t^\star \neq \emptyset$ and $\gC_t^- \neq \emptyset$; equivalently, $\gC_t$ contains at least one viable prefix and at least one non-viable prefix.
At step $t$, the greedy PRM-guided search selects the trajectory with the maximum estimated score; that is,
\begin{equation} \label{eq:greedy_selection}
    \widehat{\tau}^{(t)}_{\mathrm{greedy}}
\in
\arg\max_{\tau^{(t)} \in \gC_t} \widehat V(\tau^{(t)}), 
\end{equation}
which is a solution to \eqref{eq:goal_prm} when $\hat{V}(\tau) = V(\tau)$ for all $\tau \in \gC_t$.

To study the intrinsic difficulty of the PRM-guided search, we analyze the lower bound of the failure mode of $\widehat{\tau}^{(t)}_{\mathrm{greedy}}$. 
To this end, we define the set of non-viable prefixes that outrank \emph{all} viable prefixes under the learned verifier as
\begin{equation}
    K_t
    \triangleq \left\{ \tau_- \in \gC_t^- : \widehat V(\tau_-)
    > \max_{\tau_\star \in \gC_t^\star} \widehat V(\tau_\star)
    \right\}.
\end{equation}
Then, the greedy search selects a non-viable prefix (i.e., $\{\widehat\tau_{\mathrm{greedy}}^{(t)} \notin \widetilde{\gP}_t^\star  | 
\gC_t\}$) if and only if $|K_t|\ge 1$. We establish the lower bound of this event in the following.

\begin{restatable}{theorem}{greedyhacking}{\normalfont {(Greedy reward hacking via deterministic cutoff).}}
\label{thm:greedy_hacking_cutoff}
Under Assumption \ref{ass:noise} and any fixed $\lambda\in \R$, let $\Gamma_t(\gC_t) = \sup_{\tau_\star\in\gC_t^\star, \tau_-\in\gC_t^-} \{V(\tau_\star)+\beta_t(\tau_\star) -V(\tau_-)-\beta_t(\tau_-) \}$, $q_t(\lambda)= \Phi\left(\frac{\lambda}{\sigma_t}\right)^{N_t^\star}$, and $r_t(\lambda) = 1-\Phi\left( \frac{\Gamma_t (\gC_t)+\lambda}{\sigma_t} \right)$.
Then, it holds that
\begin{equation} \label{eq:greedy_bound}
\Pr\left( \widehat\tau_{\mathrm{greedy}}^{(t)} \notin \widetilde{\gP}_t^\star  \middle| 
\gC_t \right)
    \ge q_t(\lambda) \left[ 1-\left(1-r_t(\lambda)\right)^{N_t^-} \right] 
    \triangleq \ell_t(\lambda;\gC_t).
\end{equation}
Consequently, $\Pr\left( \widehat\tau_{\mathrm{greedy}}^{(t)} \notin \widetilde{\gP}_t^\star  \middle| 
\gC_t \right) \geq \sup_{\lambda \in \R} \ell_t(\lambda;\gC_t)$. 
\end{restatable}

Theorem~\ref{thm:greedy_hacking_cutoff} provides a certified failure mechanism for greedy PRM-guided search, being analogous to over-optimization in Best-of-N sampling \citep{gao2023scaling}. 
In the proof, we fix a normalized cutoff level and analyze the following event: every viable prefix stays below the cutoff, while at least one non-viable prefix exceeds it. 
On this event, the maximizer of the learned verifier must be non-viable. 
The lower bound
\begin{equation} \label{eq:lb_def}
    \ell_t(\lambda;\gC_t)     = \underbrace{\Phi \left(\frac{\lambda}{\sigma_t}\right)^{N_t^\star}}_{\text{viable-frontier term}} \cdot     \underbrace{\left[    1- \Phi \left( \frac{\Gamma_t(\gC_t)+\lambda}{\sigma_t}     \right)^{N_t^-}     \right]}_{\text{non-viable outlier term}}
\end{equation}
therefore exposes two coupled effects. 
The first term controls whether the viable frontier among $\gC_t^\star$ produces an unusually high verifier realization; the second term controls whether the much larger non-viable population produces an extreme upward realization. 
Thus, the failure is driven by the upper tail of the verifier error after maximization over $\gC_t^-$.

This perspective highlights a basic limitation of using the verifier to compensate for a weak generator. 
A weak generator simultaneously decreases $N_t^\star$ and increases $N_t^-$. 
The former reduces the number of viable prefixes that can defend the beam by producing a high score, while the latter increases the number of non-viable prefixes that can win through an extreme positive error. 
Consequently, the same verifier noise is much more dangerous when the candidate set is dominated by non-viable prefixes under a weak generator.

\paragraph{On the effective gap and the verifer noise.}
In \eqref{eq:greedy_bound}, the quantity $\Gamma_t(\gC_t)$ is a conservative proof envelope: it upper bounds the deterministic advantage that viable prefixes may have over non-viable prefixes in the cutoff argument, being a sufficient separation scale. 
When this scale is large, a non-viable prefix must overcome a large deterministic deficit before it can beat the viable frontier. 
When it is small, viable and non-viable prefixes are less indistinguishable even before stochastic noise is added, and the greedy decision becomes intrinsically fragile.

The cutoff $\lambda$ in Theorem~\ref{thm:greedy_hacking_cutoff} is an auxiliary proof parameter, so the clean way to interpret noise is through the optimized certificate. 
For fixed $\sigma_t$, we can reparameterize $\lambda=a\sigma_t$. 
Then, we get $\sup_{\lambda\in\mathbb R}\ell_t(\lambda;\gC_t)  =     \sup_{a\in\mathbb R} \Phi(a)^{N_t^\star} \left[ 1- \Phi \left( a+\frac{\Gamma_t(\gC_t)}{\sigma_t} \right)^{N_t^-}  \right].$
Thus, the verifier noise enters through the normalized gap $\gamma_t=\frac{\Gamma_t(\gC_t)}{\sigma_t}.$
For a larger $\sigma_t$, the same deterministic separation becomes less protective because $\gamma_t$ shrinks. 
Hence, the optimized lower bound worsens as the effective signal-to-noise ratio $\gamma_t$ decreases.

\subsection{Error accumulation across reasoning depth} \label{subsec:error_accumulation}
The preceding analysis conditions on a fixed candidate set $\gC_t$ and isolates a one-step verifier-selection failure. 
We now incorporate the generator, which determines the distribution of $\gC_t$, and extend the single-step failure mode across the reasoning horizon.

\begin{restatable}{theorem}{generatoravg}{\normalfont {(Generator-averaged greedy failure).}}\label{thm:generator_averaged_greedy_failure}
Let us suppose the current greedy prefix at depth $t-1$ is viable, i.e., $\widehat \tau_{\mathrm{greedy}}^{(t-1)} \in \widetilde{\mathcal P}_{t-1}^\star.$
Let $\mathcal H_t$ denote the search history before generating candidates at depth $t$. 
Let us also assume a Markovian candidate generator: conditional on $\mathcal H_t$, the $N_t$ candidate continuations are sampled independently, and each continuation remains viable with probability $p_t(\mathcal H_t) = \Pr_{s\sim \pi(\cdot\mid \widehat \tau_{\mathrm{greedy}}^{(t-1)})} \left( \widehat \tau_{\mathrm{greedy}}^{(t-1)}\oplus s \in \widetilde{\mathcal P}_t^\star \;\middle|\; \mathcal H_t \right).$
Finally, assume that for every mixed candidate set with $\mathcal C_t^\star\neq\emptyset$ and $\mathcal C_t^-\neq\emptyset$, there exists $\mathcal H_t$-measurable $\overline \Gamma_t(\mathcal H_t)$ such that $\Gamma_t(\mathcal C_t)\le \overline \Gamma_t(\mathcal H_t).$
Given Assumption \ref{ass:noise} and any fixed $\lambda\in\mathbb R$, we define $a_t(\lambda) = \Phi\left(\frac{\lambda}{\sigma_t}\right)$ and $b_t(\lambda;\mathcal H_t) = \Phi\left(\frac{\overline \Gamma_t(\mathcal H_t)+\lambda}{\sigma_t}\right).$
Then, it holds that 
\begin{equation}
    \Pr\left(\widehat \tau_{\mathrm{greedy}}^{(t)} \notin \widetilde{\mathcal P}_t^\star \;\middle|\; \mathcal H_t \right) \ge \delta_t(\lambda;\mathcal H_t),
\end{equation}
where $\delta_t(\lambda;\mathcal H_t) = \left(1-p_t(\mathcal H_t)\right)^{N_t} + \sum_{k=1}^{N_t} {N_t\choose k} p_t(\mathcal H_t)^k \left(1-p_t(\mathcal H_t)\right)^{N_t-k} a_t(\lambda)^k \left[ 1-b_t(\lambda;\mathcal H_t)^{N_t-k} \right].
$
\end{restatable}

For the generator-averaged results, we further assume that, conditional on $\mathcal H_t$, candidate
continuations are sampled independently from the generator and that verifier noises are conditionally
independent of the candidate-sampling randomness given the realized candidate set.

\begin{restatable}{corollary}{depthaccumulation}{\normalfont {(Depth-wise accumulation under Markovian generation).}}
\label{cor:depth_accumulation_generator}
Under the assumptions of Theorem~\ref{thm:generator_averaged_greedy_failure} for all $t \in [T_{\max}]$, let $\mathcal S_t = \left\{ \widehat \tau_{\mathrm{greedy}}^{(t)} \in \widetilde{\mathcal P}_t^\star \right\}$ denote the event that greedy search remains viable through depth $t$. 
We define $\delta_t^\star(\mathcal H_t) = \sup_{\lambda\in\mathbb R} \delta_t(\lambda;\mathcal H_t)$ on $\gS_{t-1}$ and $\delta_t^\star(\mathcal H_t) \triangleq 0$ on $\gS_{t-1}^c$.
Then, it holds that $\Pr(\mathcal S_{T_{\max}}) \le \mathbb E\left[ \prod_{t=1}^{T_{\max}} \left(1-\delta_t^\star(\mathcal H_t)\right) \right].$ 
\end{restatable}

Theorem~\ref{thm:generator_averaged_greedy_failure} separates one-step failure into generator coverage and verifier selection. 
The term $(1-p_t(\mathcal H_t))^{N_t}$ is the probability that the generator samples no viable continuation, while the summation averages the cutoff-based over-optimization event of Theorem~\ref{thm:greedy_hacking_cutoff} over the random number of viable candidates. 
Thus, the key point is that the same increase in width that improves generator coverage can also increase the number of non-viable opportunities for verifier hacking. 
Consequently, the benefit of increasing $N_t$ depends on whether the induced gain in viable coverage dominates the growth of the non-viable extreme-value frontier; when $p_t(\mathcal H_t)$ is small, the candidate pool remains highly imbalanced, so even sampled viable continuations can be eliminated by verifier over-optimization. 
Therefore, simply increasing test-time computes with a larger $N_t$ may not effectively address the over-optimization in PRM-guided search.

Corollary~\ref{cor:depth_accumulation_generator} shows that local risks compound multiplicatively across reasoning depth. 
Specifically, successful decoding requires the search to remain inside the viable-prefix set at every step; a single coverage or selection failure is irreversible. 
Thus, the trajectory-level survival probability is controlled by the product of step-wise survival probabilities.
Consequently, even small per-step lower bounds on failure can yield a large path-level failure probability over long horizons; if $\delta_t^\star(\mathcal H_t) \ge \underline\delta_t$ almost surely for $\gH_t$, then $\Pr\left(\exists t\le T_{\max}: \widehat\tau_{\mathrm{greedy}}^{(t)} \notin \widetilde{\gP}_t^\star \right) \ge 1-\prod_{t=1}^{T_{\max}}(1-\underline\delta_t).$
The core implication is that PRM-guided search is not assessed by its typical one-step behavior, but by its ability to avoid irreversible mistakes repeatedly across the entire reasoning trajectory.

\paragraph{Impact of beam search}
The greedy analysis extends directly to beam search. Let $\gB_t$ denote the top-$w_t$ prefixes under $\widehat V$, where $w_t=1$ recovers the greedy search. 
The greedy cutoff argument certifies failure when all viable prefixes stay below the cutoff and at least one non-viable prefix exceeds it; for width $w_t$, this must hold for at least $w_t$ non-viable prefixes. 
Therefore, as shown in Corollary \ref{thm:beam_hacking_cutoff} in Appendix \ref{appx:subsec:beam_hacking}, the one-step lower bound in \eqref{eq:greedy_bound} becomes
\begin{equation}
\Pr\left(\gB_t\cap \widetilde{\gP}_t^\star=\emptyset \mid \gC_t\right)
\ge
q_t(\lambda)\,
\Pr\left(\mathrm{Binomial}(N_t^-,r_t(\lambda))\ge w_t\right),
\end{equation}
with $q_t(\lambda),r_t(\lambda)$ as in Theorem~\ref{thm:greedy_hacking_cutoff}, recovering \eqref{eq:greedy_bound} at $w_t=1$. Beam search therefore increases the verifier-outlier tolerance from one to $w_t$, mitigating but not removing over-optimization. 
Since this bound preserves the cutoff structure, the generator-averaging and depth-accumulation arguments (Theorem~\ref{thm:generator_averaged_greedy_failure}, Corollary~\ref{cor:depth_accumulation_generator}) carry over with the binomial tail in place of the single-outlier term, leaving the path-level conclusion unchanged.

\paragraph{Implications for more sophisticated guided search.}
The lower bound in \eqref{eq:lb_def}  also clarifies what more sophisticated test-time procedures must improve. Methods such as MCTS, verifier refinement, or test-time training can help if they increase the effective separation $\Gamma_t(\gC_t)$, reduce harmful bias in the deterministic verifier component, or reduce the stochastic uncertainty $\sigma_t$ \citep{zhang2024rest, hosseini2024v, snell2024scaling}.
In the bound, these improvements act by either increasing the probability that the candidate set contains viable continuations or decreasing the non-viable exceedance probability $r_t(\lambda)$.
Nevertheless, these methods do not change the qualitative structure of the failure mode: any residual probability of overvaluing non-viable prefixes is amplified by maximization over candidates and repeated selection across steps \citep{gao2023scaling}. 
Guided decoding is therefore intrinsically challenging not because current verifiers are imperfect in a superficial sense.
It is challenging because the search procedure turns verifier imperfections into selection bias: among many prefixes, greedy maximization preferentially selects the prefix whose verifier error is most favorable \citep{thrun2014issues}.

\section{Mitigating Over-Optimization in Step-Level Guided Decoding}
In this section, we address the failure of the imperfect PRM-guided search by replacing the pure score maximization objective in \eqref{eq:greedy_selection} with a reward misspecification--robustness objective. 
In a nutshell, the robust optimization retains high-scoring trajectories that are also diverse in an abstract trajectory space, discouraging the beam $\gC_t$ from collapsing onto a narrow region of the search space where correlated errors in \(\hat V\) are most harmful.
In the following, we formalize the maximin PRM-guided search under reward misspecification and then derive a principled approximation algorithm for solving the robust optimization problem for efficient decoding.

\paragraph{Overview}
At a step-level pruning stage $t$, let the candidate trajectories $\gC_t = \{ \tau^{(t)}_1,\dots,\tau^{(t)}_N \}$ have the estimated step-level scores $s\in\mathbb{R}^N$ such that $s_i=\widehat{V}(\tau^{(t)}_i)$. 
Also, we let \(r_i = V(\tau^{(t)}_i)\) be the true score and $\nu_i = r_i - s_i$ be a signed value estimation error at $\tau^{(t)}_i$. 
For brevity, in the following, we omit superscript $(t)$ when there is no ambiguity. 
We instantiate the idea of robust step-level guidance with a repulsive penalty defined directly on pairwise trajectory similarity. To this end, we define a radial basis function (RBF) kernel 
\begin{equation} \label{eq:rbf_kernel}
k(\tau_i, \tau_j) = \exp \left( - \gamma d^2(\tau_i, \tau_j) \right),
\end{equation}
where \(\gamma>0\) is a bandwidth parameter, and
\( d(\tau_i, \tau_j) \triangleq \sqrt{ 2 - 2 \frac{\mathrm{emb}(\tau_i)^\top \mathrm{emb}(\tau_j)} {\lVert \mathrm{emb}(\tau_i) \rVert \cdot \lVert \mathrm{emb}(\tau_j) \rVert} }, \)
with \(\mathrm{emb}(\cdot)\) denoting a representation (e.g., the last hidden state of the PRM).
Here, the kernel takes a larger value if $\tau_i$ and $\tau_j$ have a small (angular) distance in the representation space from $\mathrm{emb}(\cdot)$. 
Also, the kernel matrix $\{ K_{ij} \triangleq k(\tau_i, \tau_j) \}$ is a symmetric positive semidefinite kernel, \(K_{ii}=1\), and \(0 \le K_{ij} \le 1\) for all \(i,j\) by construction.
To avoid manual tuning, we use the median heuristic \citep{gretton2012kernel, garreau2017large} that sets $\gamma$ by the inverse of the median pairwise distance.

With the RBF kernel, we seek to retain $m$ trajectories (e.g., $m = w_t$) by solving 
\begin{equation}
\label{eq:overview}
\max_{q \in \{0,1\}^N} \quad J(q)  \triangleq s^\top q - \tilde{B} \sqrt{q^\top K q} \qquad \text{s.t.} \qquad \mathbf{1}^\top q = m,
\end{equation}
where $q_i=1$ indicates that $\tau_i$ is selected. The first term recovers standard step-level guidance, while the second penalizes redundant selections. In particular, $\tilde{B}=0$ reduces to PRM-guided beam search, whereas larger $\tilde{B}$ increasingly favors geometrically diverse trajectories. 
Throughout the paper, we let the feasible set $\gQ \triangleq \left\lbrace q \in \{0, 1\}^N : \mathbf{1}^\top q = m \right\rbrace$. 
Finally, our framework changes only the score optimization mechanism, while leaving others fixed. 
This makes it compatible with better score estimation via explicit online training \citep{zhang2024rest, hosseini2024v, cui2025process} or sophisticated search techniques \citep{feng2023alphazero, chen2024alphamath}.

\subsection{Distributional robustness perspective} \label{subsec:kernel_persp}
Our goal is to identify structural assumptions under which controlling \(\nu^\top q = r^\top q - s^\top q\) leads to a kernel-based regularizer that controls the sensitivity of the selected support to estimation error.
Let \(\Phi_k\) denote the reproducing kernel Hilbert space (RKHS) associated with the kernel \(k\).
Suppose the estimation error is induced by a function \(e(\cdot)\in\Phi_k\) such that $\nu_i = e(\tau_i)$ with $\|e\|_{\Phi_k}\le B$ for some radius \(B>0\).
For common kernels such as RBF kernels, this can be interpreted as a smoothness assumption: prefixes that are close under the kernel tend to have similar verifier errors.
This assumption is a reasonable local approximation for modern transformer-based reward models, and similar assumptions have been widely studied
\citep{malladi2023kernel,zhang2024overcoming}.

This RKHS model is not separate from the verifier regularity condition in Assumption~\ref{ass:noise}; rather, it refines the structured bias component. In Assumption~\ref{ass:noise}, the deterministic verifier bias is
controlled only through a pointwise bound \(|\beta_t(\tau)|\le B_t\). By contrast, an RKHS model
places geometric structure on this bias. More generally, we can decompose the signed verifier
error as $\nu_i = r_i-s_i = \tilde{e}(\tau_i)+\zeta_i,$ 
where \(\tilde{e} \in\Phi_k \) with $\|\tilde{e}\|_{\Phi_k}\le B$ is a smooth structured bias component and \(\zeta_i\) is a residual stochastic component. 
Then, the reproducing property implies the pointwise bound
\(|\tilde{e}(\tau_i)| \le \|\tilde{e}\|_{\Phi_k}\sqrt{k(\tau_i,\tau_i)} \le B\) when \(K_{ii}=1\).
Thus the RKHS model satisfies the bounded-bias and Gaussian-noise requirements of Assumption~\ref{ass:noise}. The difference is that Assumption~\ref{ass:noise} only uses the crude pointwise control \(|\beta|\le B\), whereas the RKHS model further constrains how the bias varies across trajectories. 
Thus, the RKHS model strengthens the bounded-bias component of Assumption \ref{ass:noise} by requiring verifier errors to vary smoothly over the trajectory representation space, but it does not impose the i.i.d. random noise assumption in Assumption \ref{ass:noise}.
In particular, nearby trajectories under the kernel must have related verifier errors, allowing the model to capture correlated misspecification. The following result exploits this additional structure to obtain a geometry-aware robustness penalty.

\begin{restatable}{theorem}{rkhsworstcase}{\normalfont {(Worst-case aggregate error under RKHS smoothness).}} \label{thm:rkhs_worst_case}
If the estimation error is induced by a function \(e\in\Phi_k\) with \(\|e\|_{\Phi_k}\le B\), then for every \( {q \in \{0,1\}^N} \),
\begin{equation}
\label{eq:rkhs_bound}
\sup_{\|e\|_{\Phi_k}\le B}
\left|
\sum_{i=1}^N q_i e(\tau_i)
\right|
=
B \sqrt{q^\top K q}.
\end{equation}
\end{restatable}

Theorem~\ref{thm:rkhs_worst_case} shows that, under the RKHS smoothness condition, the kernel regularizer is the \emph{worst-case estimation error radius} of the aggregate decision \(q\) under a smoothness model on the error function. 
Consequently, it yields $r^\top q \ge s^\top q -B\sqrt{q^\top K q} = J(q) $ for all admissible error functions, yielding the natural robust optimization objective in \eqref{eq:overview}.
Thus, the kernel penalty measures how vulnerable the decision is to smooth but adversarial score misspecification.
In this sense, our formulation can be viewed as an instance of a distributionally robust optimization (DRO)-inspired robust decision rule: the selected subset is evaluated by its worst-case aggregate reward under an RKHS-bounded ambiguity set over verifier errors, analogous in spirit to ambiguity-set based DRO and kernel DRO formulations \citep{duchi2021learning, rahimian2022frameworks, zhu2021kernel}.

\begin{figure}
    \centering
    \includegraphics[width=0.75\linewidth]{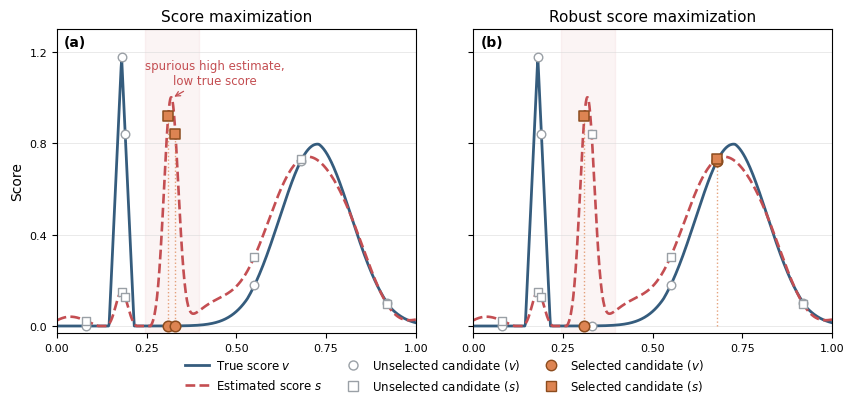}
    \caption{\textbf{Robust subset selection under score misspecification $(N=8, m=2)$.} The true score has a sharp and a broad peak, while the estimate introduces a spurious high-score region near the sharp peak. Maximizing $s$ selects two nearby points in this erroneous region. Adding diversity regularization spreads the selections, recovering a higher-value point from the broad peak and improving robustness.}
    \label{fig:score_maxim}
\end{figure}

Figure \ref{fig:score_maxim} illustrates that selecting many nearby points makes the aggregate estimate fragile, because smooth errors can align coherently with a large value of $q^\top K q$ across the support. 
Therefore, whenever the score estimator overestimates an entire local region of the feature space as illustrated in the red regime in Figure \ref{fig:score_maxim}, the optimizer would concentrate mass on many nearby points, thereby compounding the estimation error.
On the other hand, the regularized objective function as in \eqref{eq:overview} explicitly discourages such coherence; it forces the optimizer to trade off raw estimated reward against uncertainty induced by correlated misspecification. 
From this perspective, repulsion is a form of distributional robustness over feature space, remaining reliable under plausible smooth perturbations of the reward model.
Theorem \ref{thm:clustered_failure_rescue} in Appendix \ref{appx:kernel_narrow_failure} establishes the case when the kernel penalty mitigates the failure mode of the beam search studied in Corollary \ref{thm:beam_hacking_cutoff} in which the PRM beam collapses onto a concentrated group of non-viable prefixes when their PRM scores are only marginally better than a viable alternative.

\paragraph{Empirical choice of the RKHS radius \(B\).}
In Theorem~\ref{thm:rkhs_worst_case}, the constant \(B\) upper-bounds the RKHS norm of the unknown error function \(e\), i.e., \(\|e\|_{\Phi_k}\le B\), and estimating such a radius from data is known to be delicate \citep{maddalena2021deterministic, scharnhorst2022robust, tokmak2025safe}.
While it is typically treated as a hyperparameter, we use a calibration-based scaling rule to set its value in practice, when calibration data with both predicted and reference scores is available.
Given  $\nu_i = r_i - s_i$, we find a variational formulation such that $\tilde{B}$ is the RKHS norm of the minimum-norm function \(\tilde{e} \in\Phi_k\) satisfying \(\tilde{e}(\tau_i)=\nu_i\) for all \(i\) with $\nu = (\nu_1, \cdots, \nu_N)$.
This perspective yields a natural sample-based proxy  $\tilde{B} \triangleq \sqrt{\nu^\top K^\dagger \nu},$ which is the solution to $\tilde{B} \in \min_{\tilde{e} \in\Phi_k:  \tilde{e}(\tau_i)=\nu_i \ \forall i} \|\tilde{e}\|_{\Phi_k},$ with $K^\dagger$ being the Moore–Penrose inverse of $K$.
Thus, \( \tilde{B} \) is the smallest RKHS radius consistent with the observed error pattern on the sample, which could be multiplied by a constant factor in practice for preventing underestimation of off-sample error complexity.
When oracle step rewards are unavailable, we use the maximum within-trajectory variation in estimated step rewards as a proxy for the calibration residual magnitude. This heuristic is intended to scale the robust penalty to the observed instability of the verifier, rather than to provide a consistent estimate of the true RKHS error norm.

\paragraph{Connection to diversity-based regularized decoding.}
Although the regularizer in~\eqref{eq:overview} superficially resembles diversity-promoting decoding such as diverse beam search, its role is different: it arises as a robustness penalty against smooth PRM misspecification rather than as a generic lexical or semantic diversity heuristic. In particular, $\sqrt{q^\top Kq}$ controls the worst-case aggregate verifier error under the RKHS error model, whereas minimum Bayes risk (MBR) uses candidate geometry for consensus or centrality \citep{freitag2022high,eikema2022sampling,jinnai2025regularized}, determinantal point process-style selection promotes volume or low redundancy \citep{kulesza2012determinantal,meister2021determinantal}, and diverse beam search encourages lexical or semantic mode coverage during decoding \citep{vijayakumar2016diverse,shi2025semantic}. 
Further, this perspective provides a principled mechanism to select $\tilde{B}$ from the calibration set, unlike the previous regularized decoding methods.
We provide a detailed comparison in Appendix~\ref{app:diversity_comparison}.

\subsection{Method}
\label{sec:method}
\paragraph{Overview}
The main difficulty in~\eqref{eq:overview} is the non-separable term \(\sqrt{q^\top K q}\). 
Our approach removes this nonlinearity by a variational problem with an auxiliary scalar \(\eta > 0\). 
The variational reformulation yields a family of surrogate objectives that are submodular under a cardinality constraint, which can be solved by fast principled approximate algorithms. 
In this way, we have a principled approximate solution with  a controllable approximation error for the \emph{original} square-root objective.
Algorithm \ref{alg:grid-submodular} describes the full procedure, and we refer to Appendix \ref{appx:sec:details} for the exact configurations.

\begin{algorithm}[t]
\caption{Grid-based submodular surrogate maximization for maximin PRM-guided search}
\label{alg:grid-submodular}
\begin{algorithmic}[1]
\REQUIRE scores \(s \in \mathbb{R}^N\), kernel matrix \(K \in \mathbb{R}^{n \times n}\), penalty \(\tilde{B} > 0\), cardinality \(m\), grid parameter \(\xi > 0\), inner solver \(\mathcal{A}\)
\STATE Construct the grid $\gG_\xi = \left\{\eta_t = \sqrt{m}(1 + \xi)^h : h = 0,1,\dots,H
\right\}, \;
H \triangleq \left\lceil \log_{1 + \xi} \sqrt{m} \right\rceil$
\FOR{each \(\eta_i \in \gG_\xi\)}
    \STATE Compute \(q_i^\dagger \leftarrow \mathcal{A}(\eta_i)\), where \(\mathcal{A}(\eta_i)\) approximately solves $\max_{q \in \gQ} \left[ 
        s^\top q - \frac{\tilde{B}}{2\eta_i} q^\top K q
        \right]$
\ENDFOR
\STATE Return $q^\dagger \in \arg\max_{i \in \{0,\dots, H \}}
\left[ s^\top q_i^\dagger - \tilde{B} \sqrt{{q_i^\dagger}^\top K q_i^\dagger} \right]$
\end{algorithmic}
\end{algorithm}

\paragraph{Variational reformulation}
We use the identity $\sqrt{z} = \min_{\eta>0} \left( \frac{z}{2\eta} + \frac{\eta}{2} \right)$ for $z \ge 0.$
Applying this identity to $z = q^\top K q$ yields the exact representation 
\begin{equation}
\label{eq:J-eta-exact}
J(q) = \max_{\eta>0} J_\eta(q),
\qquad
J_\eta(q) \triangleq s^\top q - \frac{\tilde{B}}{2\eta} q^\top K q - \frac{\tilde{B}}{2}\eta,
\end{equation}
yielding $\max_{q \in \gQ} J(q) = \max_{q \in \gQ} \max_{\eta>0} J_\eta(q) = \max_{\eta>0} \max_{q \in \gQ} J_\eta(q).$

\paragraph{Outer optimization: Grid search over the auxiliary variable}
Note that, for any set $q$, the maximizing auxiliary variable in~\eqref{eq:J-eta-exact} is $\eta = \sqrt{q^\top K q}$, which is in $[\sqrt{m},  m]$ for every feasible \(q \in \gQ \). 
Therefore, we discretize this interval with a multiplicative grid:
\begin{equation}
\label{eq:eta-grid}
\gG_\xi = \left\{\eta_i = \sqrt{m}(1 + \xi)^h : h = 0,1,\dots,H
\right\}, \qquad
H \triangleq \left\lceil \log_{1 + \xi} \sqrt{m} \right\rceil .
\end{equation}

In \eqref{eq:eta-grid}, we note that the number of surrogate problems is $|\gG_\xi| = H+1 = O \left(\frac{\log m}{\xi}\right),$ which increases logarithmically in \(m\) and inversely in \(\xi\).
For each \(\eta\in\mathcal G_\xi\), we solve the surrogate problem $\max_{q\in\mathcal Q} \left[ s^\top q-\frac{\tilde B}{2\eta}q^\top Kq \right],$ where the constant term \(-\tilde B\eta/2\) is omitted because it does not affect the optimizer for fixed \(\eta\). This produces a candidate solution \(q_\eta^\dagger\). After all grid points are processed, we select among these candidates using the original square-root objective $J(q)=s^\top q-\tilde B\sqrt{q^\top Kq}.$ Thus, the variational surrogates are used only to generate candidate subsets, while the final outer selection is made according to the true robust objective. Equivalently, for each \(q_\eta^\dagger\), the auxiliary variable is analytically optimized out by evaluating \(J(q_\eta^\dagger)=\max_{\eta'>0}J_{\eta'}(q_\eta^\dagger)\).

\paragraph{Inner optimization: Principled approximate algorithm under submodularity} 
For a fixed \(\eta>0\), the surrogate problem $\max_{q \in \gQ} J_\eta(q) \triangleq s^\top q - \frac{\tilde{B}}{2\eta} q^\top K q - \frac{\tilde{B}}{2}\eta$ is a cardinality-constrained non-monotone submodular optimization problem.
We let $1_j \in \{0, 1 \}^N$ be the unit vector whose $j$-th element is 1 and remaining elements are 0. 
To obtain a practical solver with a formal approximation guarantee, we use a two-stage procedure:
\begin{enumerate}
    \item \textbf{Guaranteed warm start:} We run a randomized greedy algorithm for non-monotone submodular maximization under a cardinality constraint. We apply the randomized greedy algorithm of \citet{buchbinder2014submodular}: Starting from \(q_0=\varnothing\), for rounds \(r=1,\dots,m\), the algorithm computes the marginal gains
    \begin{equation}
    \Delta_r(j) \triangleq J_\eta(q_{r-1} + 1_j )-J_\eta(q_{r-1}), \qquad j \in \{i \in [N] :  \{q_{r-1} \}_i = 0 \},
    \end{equation}
    and forms the set \(M_r\) of the \(m\) elements with largest marginal gains among unselected trajectories (or all remaining elements if fewer than \(m\) remain).
    Then, it updates $q_r \leftarrow q_{r-1} + 1_{\iota(r)}$ by uniformly sampling $\iota(r)$ from the set $M_r$. 
    As a result, this step produces a feasible $q_m \in \gQ$.
    \item \textbf{Objective-improving post-processing:} We apply improving pairwise exchanges until a
    \(1\)-swap local optimum is reached. Such exchange-based local search is a standard tool in
    constrained non-monotone submodular maximization, particularly under cardinality and matroid
    constraints \citep{lee2009non,fadaei2011maximizing}. To this end, we use the gain of replacing active trajectory $i$ with inactive trajectory $j$ given by
    \begin{equation} \label{eq:swap-gain}
    \Delta_{i \to j}(q) = s_j - s_i -
    \frac{\tilde{B}}{\eta} \sum_{k \in [N] \setminus \{i\}: \{ q \}_{k} = 1 } (K_{jk} - K_{ik}).
    \end{equation}
\end{enumerate}

In the two-stage optimization algorithm, the first stage provides a provable approximation bound for each surrogate problem. In particular, for a general non-monotone submodular maximization, the randomized greedy algorithm finds $q_m$ such that $\E [J_\eta (q_m)] \geq \rho \max_{q \in \gQ} J_\eta (q)$ with $\rho \in (0, 1]$ (see \citet{buchbinder2014submodular} for the exact constant). 
The pairwise-swap stage is a monotone post-processing step that monotonically improves the variational objective. Given the finite feasible family $q \in \gQ$, since a cyclic search is impossible, it must terminates after finitely many improving swaps. 
The resulting solution is a 1-swap local optimum by construction. 
Here, for dense \(K\), the residual randomized greedy initialization runs in \(O(Nm)\), with \(m\) rounds of \(O(N)\) marginal updates and partial selection. 
To make the total running time polynomial, we cap the number of 1-swap refinement steps by \(30\).

\paragraph{Error bound}
We now state the main guarantee for the error of Algorithm~\ref{alg:grid-submodular} compared to the original square-root objective \(J\). 
The result separates the error into two terms: the approximation error induced by the multiplicative grid over \(\eta\), and the approximation error of the randomized greedy inner solver.
For simplicity, we consider the case that $J_\eta(q) \ge 0$ for all $q \in \gQ$ and $\eta \in [\sqrt{m}, m]$ as submodular maximization algorithms are stated for nonnegative objectives, and extend it to a general objective in Corollary \ref{cor:shifted-end-to-end} in Appendix \ref{appx_subsec:corollary}.

\begin{restatable}{theorem}{endtoendanalysis}{\normalfont{(End-to-end approximation guarantee).}}  \label{thm:end-to-end}
Let \(q^\star \in \arg\max_{q \in \gQ} J(q)\), and let \(q^\dagger\) be the output of Algorithm~\ref{alg:grid-submodular}. 
Let us assume that for every \(\eta \in \gG_\xi\), the randomized greedy warm start returns a set \(q_\eta^{\mathrm{rg}}\) satisfying $\E \left[J_\eta \left(q_\eta^{\mathrm{rg}}\right)\right]
\ge \rho \max_{q \in \gQ} J_\eta(q)$ and that the subsequent pairwise-swap local search returns \(q_\eta^\dagger\) with $J_\eta(q_\eta^\dagger)\ge J_\eta \left(q_\eta^{\mathrm{rg}}\right).$
Then, it holds that
\begin{equation}
\label{eq:end-to-end-main}
\mathbb E[J(q^\dagger)] \ge \rho \left( J(q^\star) - \frac{\tilde{B} m \xi^2}{2(1 + \xi)} \right).
\end{equation}
\end{restatable}

Theorem~\ref{thm:rkhs_worst_case} shows that $J(q)=s^\top q-\tilde B\sqrt{q^\top Kq}$ is a certified lower bound on the true aggregate score \(r^\top q\) under RKHS-bounded verifier
misspecification. Theorem~\ref{thm:end-to-end} then shows that Algorithm~\ref{alg:grid-submodular} approximately maximizes this certificate, up to the inner-solver factor \(\rho\) and the grid error \(O(\tilde Bm\xi^2/(1+\xi))\).
This guarantee illustrates why adding the kernel penalty is effective against over-optimization. 
Maximizing \(s^\top q\) alone does not control how much of the selected score may be due to possibly coherent verifier error; a spuriously overestimated local cluster can therefore dominate the entire beam (cf. Figure \ref{fig:score_maxim}). 
The robust objective prevents this within the RKHS uncertainty class by subtracting exactly the worst-case aggregate error radius of the selected set. 
Thus, a set is retained only when its nominal PRM advantage remains positive after charging for its exposure to structured verifier misspecification.

\section{Experiments}
We evaluate maximin search on a series of mathematical reasoning benchmarks: GSM8K \citep{cobbe2021training}, MATH-500 \citep{hendrycks2021measuring,lightman2023let}, AIME'24\footnote{\url{https://huggingface.co/datasets/HuggingFaceH4/aime_2024}}, and AIME'25\footnote{\url{https://huggingface.co/datasets/MathArena/aime_2025}}.
The benchmarks are ranging from grade-school level to competition level problems, and report accuracy across all benchmarks. For AIME datasets, which contain only 30 samples, we use three different random seeds and report the average accuracy.

\paragraph{Base models and configuration.}
We use off-the-shelf language models as base generators, including Qwen2.5-Math-1.5B \citep{yang2024qwen25mathtechnicalreportmathematical} and Phi-3.5-mini-instruct \citep{abdin2024phi}, together with the Skywork o1-PRM-1.5B reward model \citep{he_2024_16998085}. 
These models are chosen for two complementary reasons. Different generator families test whether the method depends on a particular generator, while their reasoning strength is high enough to produce viable continuations but low enough to leave many non-viable candidates, making step-level search decisions non-degenerate.
All models are used without fine-tuning, online adaptation, or task-specific training. 
For a generated trajectory $\tau$, the PRM assigns a reward to each reasoning step.
We follow \citet{he_2024_16998085} by using the average step reward as the trajectory-level value estimate $\hat{V}(\tau)$. 
We use zero-shot chain-of-thought prompts for all generators, with full prompt templates provided in Appendix \ref{appx:sec:details}. We follow the standard generation configuration described in Appendix \ref{appx:sec:details}. The hyperparameters and solver settings for the nested optimization in maximin search are provided in Appendix \ref{appx:sec:details}.

\paragraph{Baselines.}
We evaluate all methods under matched inference budgets of $N \in \{16, 64\}$ trajectories. The budget is defined as the maximum number of parallel trajectories maintained during decoding: the beam width for beam search, the sample size $N$ for Best-of-N methods, and $4m$ for PRM-guided search, where we use a fixed branch factor of $4$ and set the search width to $m=\lfloor N/4 \rfloor$. We compare maximin search against five baselines spanning three classes.

\begin{itemize}
    \item \textbf{No external supervision:} greedy decoding and beam search.
    Greedy decoding selects the highest-probability token at each step, while beam search keeps and expands the top-$N$ partial trajectories by model likelihood.

    \item \textbf{Outcome-level supervision:} Best-of-N (BoN) and MBR-BoN~\citep{jinnai2025regularized}.
    BoN selects the trajectory with the highest estimated outcome score among $N$ samples, while MBR-BoN selects the one maximizing
    $\hat{V}(\tau) + 0.1 \cdot \frac{1}{N}\sum_{i=1}^N U(\tau,\tau_i)$, where     $U(\tau,\tau') = \mathrm{cossim}(\mathrm{emb}(\tau), \mathrm{emb}(\tau'))$.
    \item \textbf{Step-level supervision:} SBS~\citep{chen2024alphamath}, a PRM-guided beam search that keeps the top-$m$ partial trajectories by PRM score.
\end{itemize}

\begin{table}[t]
\centering
\caption{
\textbf{
Benchmark results under different generators and inference budgets.}
Best results within each generator--dataset--budget block are bolded; second-best results are underlined. Greedy uses $N=1$.
We report the standard deviations in Table \ref{tab:aime_std_results} in Appendix \ref{appx:sec:details}.}
\label{tab:qwen25_phi_results}
\small
\setlength{\tabcolsep}{4pt}
\newcommand{\tb}[1]{{\fontseries{b}\selectfont #1}}
\begin{tabular}{llcccccccc}
\toprule
\multirow{2}{*}{Generator} & \multirow{2}{*}{Method}
& \multicolumn{2}{c}{GSM8K}
& \multicolumn{2}{c}{MATH-500}
& \multicolumn{2}{c}{AIME'24}
& \multicolumn{2}{c}{AIME'25} \\
\cmidrule(lr){3-4}
\cmidrule(lr){5-6}
\cmidrule(lr){7-8}
\cmidrule(lr){9-10}
& & $N=16$ & $N=64$ & $N=16$ & $N=64$ & $N=16$ & $N=64$ & $N=16$ & $N=64$ \\
\midrule
\multirow{6}{*}{Qwen2.5-Math-1.5B}
& Greedy
& 82.3 & 82.3
& 70.4 & 70.4
& 10.0 & 10.0
& 10.0 & 10.0 \\

& Beam search
& 84.3 & 84.9
& 72.2 & 75.2
& \tb{16.7} & \tb{16.7}
& {10.0} & 10.0 \\

& BoN
& 87.3 & 88.4
& \underline{75.2} & 76.8
& \tb{16.7} & \tb{16.7}
& \tb{15.6} & \tb{17.8} \\

& BoN-MBR
& \underline{88.0} & \underline{88.8}
& \underline{75.2} & \underline{77.4}
& \underline{14.4} & \tb{16.7}
& \tb{15.6} & \underline{12.2} \\

& SBS
& 87.4 & 87.5
& 72.4 & 76.6
& 8.9 & \underline{15.6}
& \underline{13.3} & \tb{17.8} \\

& Maximin
& \tb{88.9} & \tb{89.3}
& \tb{76.8} & \tb{78.8}
& \tb{16.7} & \tb{16.7}
& \tb{15.6} & \tb{17.8} \\

\midrule
\multirow{6}{*}{Phi-3.5-mini-instruct}
& Greedy
& 84.1 & 84.1
& 45.8 & 45.8
& 3.3 & 3.3
& 0.0 & 0.0 \\

& Beam search
& 85.8 & 84.7
& 53.6 & 53.0
& 3.3 & 3.3
& \tb{6.7} & {0.0} \\

& BoN
& 89.3 & 89.6
& \underline{63.4} & 65.8
& \underline{6.7} & 5.6
& \underline{2.2} & 3.3 \\

& BoN-MBR
& \underline{89.7} & 89.7
& 62.2 & 65.0
& \underline{6.7} & 6.7
& \underline{2.2} & 3.3 \\

& SBS
& 88.9 & \tb{91.4}
& 63.2 & \underline{66.4}
& 4.4 & \underline{7.8}
& \underline{2.2} & \underline{6.7} \\

& Maximin
& \tb{89.8} & \underline{91.2}
& \tb{63.8} & \tb{70.2}
& \tb{8.9} & \tb{8.9}
& \underline{2.2} & \tb{7.8} \\
\bottomrule
\end{tabular}
\end{table}

\subsection{Benchmark results}

Table \ref{tab:qwen25_phi_results} shows that naively using PRM scores as step-level guidance is not sufficient. 
As expected, external guidance substantially improves over unguided decoding; for example, on MATH-500 with Phi-3.5-mini-instruct at $N=64$, beam search obtains 53.0\%, while BoN/SBS/maximin reach 65.8\%/66.4\%/70.2\%.
However, SBS does not consistently outperform outcome-level selection methods such as BoN or BoN-MBR. 
For instance, with Qwen2.5-Math-1.5B, SBS is below BoN-MBR on GSM8K at $N=64$ (87.5 vs. 88.8), on MATH-500 at $N=16$ (72.4 vs. 75.2), and on AIME’24 at $N=16$ (10.0 vs. 13.3). 
This is consistent with prior work showing that direct step-level aggregation can be brittle \citep{cinquin2025limits}.

However, adding the quadratic kernel penalty through maximin search makes the same step-level PRM signal more effective. Maximin achieves the best or tied-best result in 14 of the 16 generator–dataset–budget blocks, and improves over SBS by roughly 17\%–35\% on average across generator and search-budget settings, with only a 0.27\% increase in runtime overhead when $N=16$. The gains are more pronounced on harder benchmarks; for example, on Phi-3.5-mini-instruct, maximin improves over SBS from 66.4\% to 70.2\% on MATH-500 at $N=64$, and from 6.7\% to 13.3\% on AIME’24 at $N=64$. 
Consistent with this trend, Figure \ref{fig:perf_by_difficulty} shows that on MATH-500 the methods perform similarly on easier levels, while maximin obtains the highest Level-5 accuracy for both $N=16$ and $N=64$. 
These results suggest that PRMs contain useful process-level information, but that robust aggregation is important for exploiting it. This improvement is obtained without online adaptation or fine-tuning, using the same off-the-shelf generators and PRMs.


\begin{figure}
    \centering
    \includegraphics[width=0.95\linewidth]{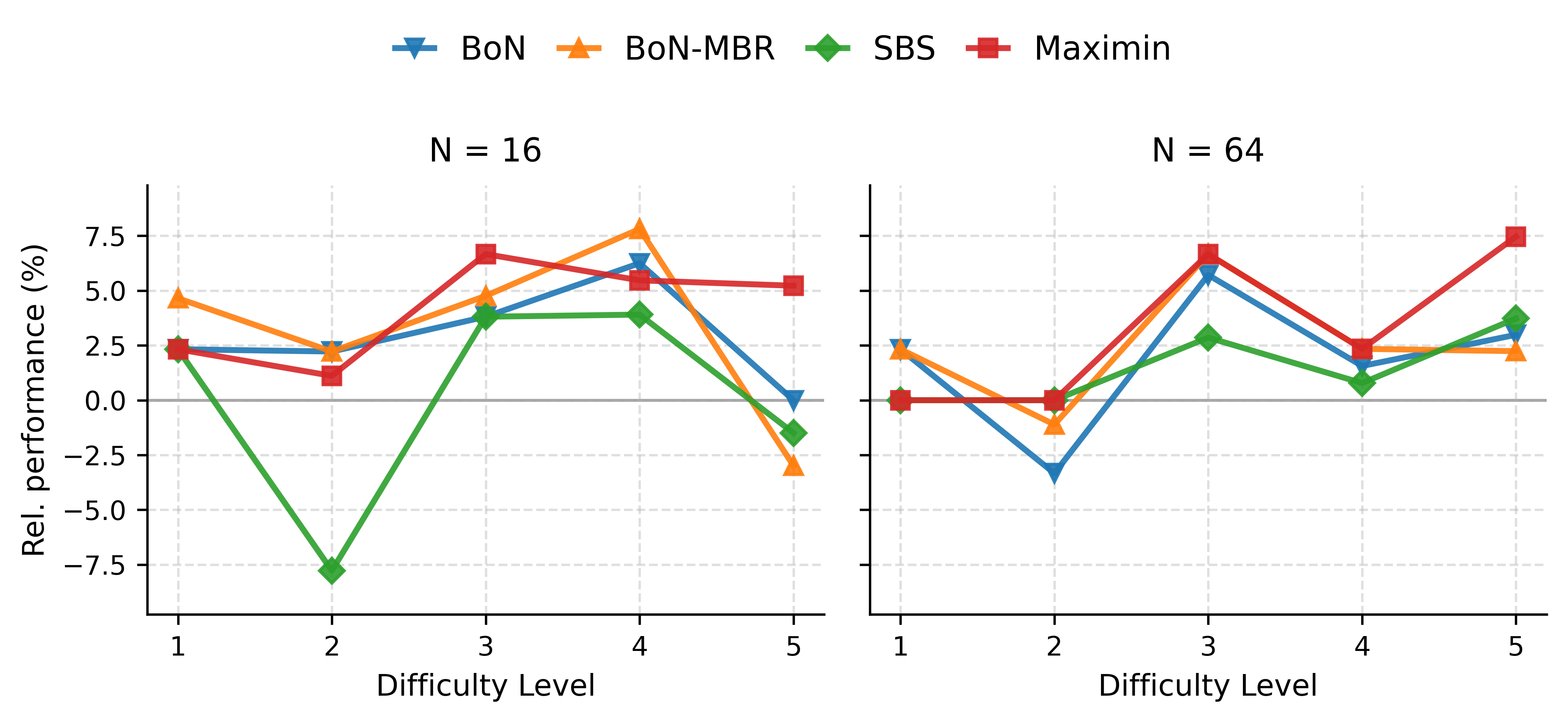}
    \caption{\textbf{Performance of various decoding strategies relative to beam search across difficulty levels (1–5) at $N=16$ (left) and $N=64$ (right) on MATH-500.} Series show the accuracy difference (in percentage points) from beam search; positive values indicate improvement over the baseline.}
    \label{fig:perf_by_difficulty}
\end{figure}

\begin{figure}[t]
    \centering
    \includegraphics[width=0.5\textwidth]{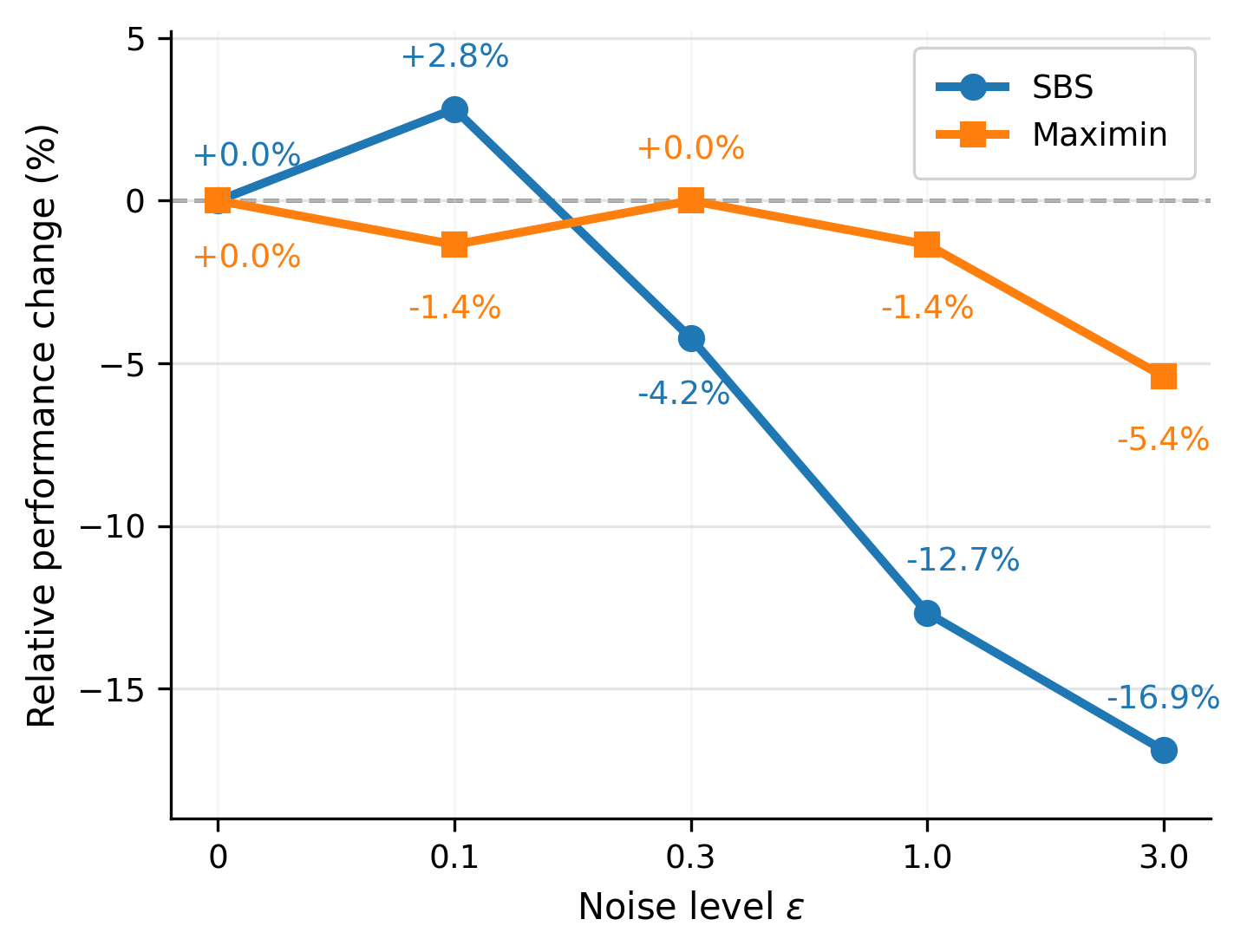}
    \caption{\textbf{Relative performance change (\%) under different levels of noise injections on verifier scores.} 
    }
    \label{fig:noise_impact}
\end{figure}

\subsection{Impacts of reward noise} \label{subsec:impacts_noise}
We test the robustness of SBS and maximin search to noisy reward estimates by adding independent Gaussian noise $\gN(0,\epsilon^2)$ to the process reward, with $\epsilon \in \{0.1,0.3,1.0,3.0\}$. 
All experiments use Qwen2.5-Math-1.5B guided by Skywork o1-PRM-1.5B on 100 random samples from MATH-500. 
On this subset, maximin and SBS achieve $79\%$ and $76\%$ accuracies, respectively. 
Although these numbers differ slightly from the full evaluation due to subsampling, the relative trends are consistent.

Figure~\ref{fig:noise_impact} shows that both methods degrade as the noise level increases, but maximin search is generally less sensitive to score perturbations. 
In particular, at $\epsilon=1.0$, SBS shows a $12.7\%$ relative degradation from its noise-free baseline, while maximin search changes by only $1.4\%$.
Even for extreme noise level at $\epsilon=3.0$, maximin search degrades by only $5.4\%$ as opposed to $16.9\%$ in SBS.
Interestingly, a small amount of noise improves SBS, increasing performance from $71\%$ to $73\%$ at $\epsilon=0.1$. This suggests that mild perturbations may regularize SBS by preventing over-reliance on small score differences. However, this benefit disappears once the perturbation becomes large enough to substantially corrupt the score ordering.
Overall, maximin search remains more robust because it does not depend solely on the noisy score estimation, but also preserves a less concentrated reasoning trajectories.

\begin{figure}[t]
    \centering
    \includegraphics[width=0.5\textwidth]{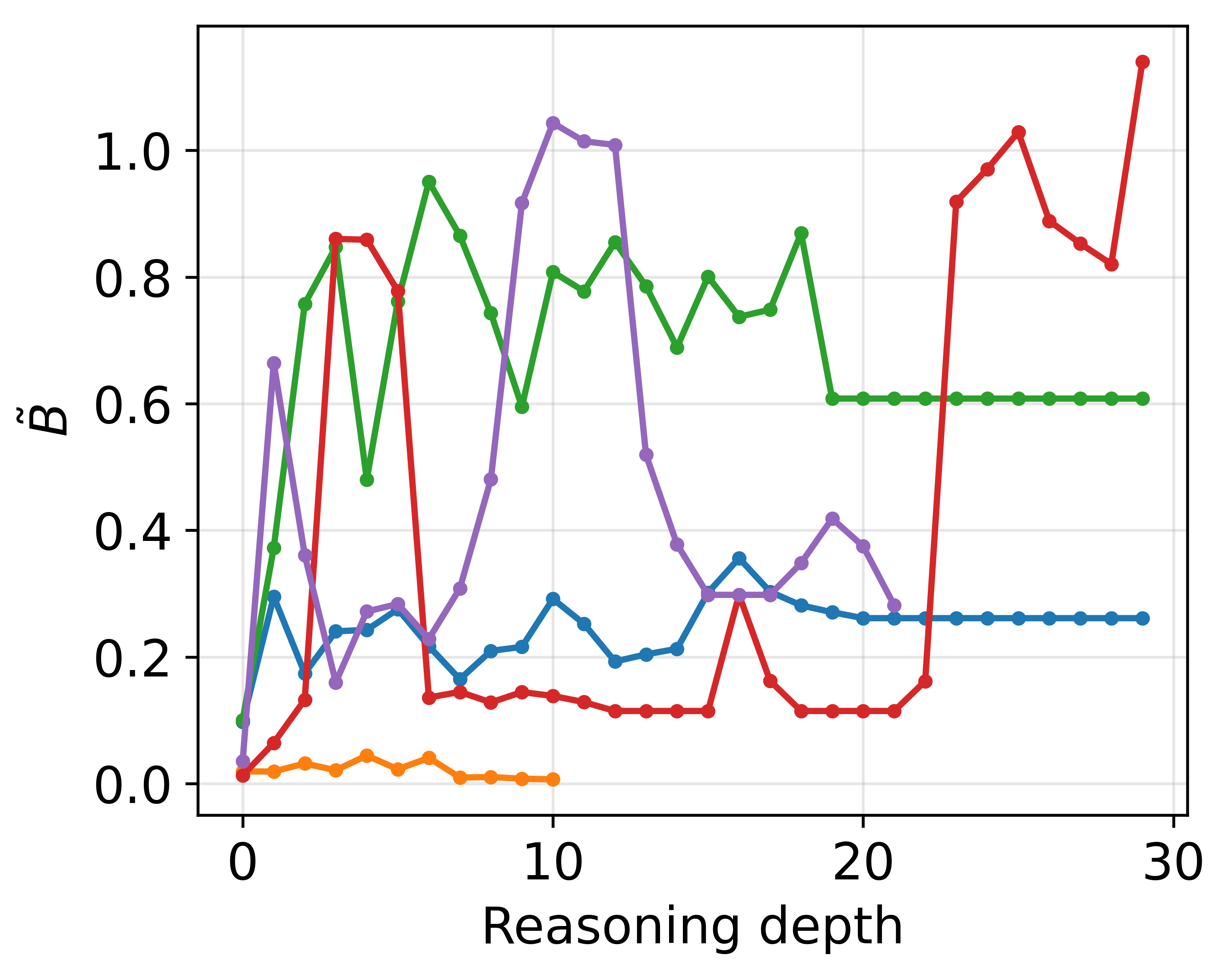}
    \caption{\textbf{Estimated $\tilde{B}$ across depth.} Each line corresponds to a different example. The estimated RKHS radius substantially different for each example. Across the reasoning depth within the same example, the radius estimate varies and does not exhibit a clear pattern.}
    \label{fig:b_estim}
\end{figure}

\subsection{Algorithmic design choices} \label{subsec:algo_design_choice}
With the setup same as in \S \ref{subsec:impacts_noise}, we examine three major algorithmic choices in maximin PRM-guided search: the local improvement step, the adaptive RKHS radius, and the representation geometry used by the RBF kernel.

\paragraph{Local improvement validates the robust objective.}
We introduce a 1-swap local improvement step to further optimize the robust objective in \eqref{eq:overview} after the initial randomized greedy selection.
Removing this step reduces accuracy from $79\%$ to $77\%$, only slightly above SBS.
This is notable because naive PRM-guided search can suffer from reward over-optimization: directly maximizing imperfect PRM scores may amplify score errors and lead to reward hacking (cf. \S\ref{sec:bad_step_guide}).
In contrast, further optimizing our robust objective improves performance.
Thus, the robust objective changes the role of optimization: rather than exacerbating PRM errors, additional optimization becomes beneficial.
This suggests that the maximin objective is a more reliable target for PRM-guided decoding under imperfect reward estimates.

\paragraph{The uncertainty scale should be adaptive, not fixed.}
We set the RKHS radius $\tilde{B}$ in \eqref{eq:overview} using the sample-dependent proxy
$\tilde{B}=\sqrt{\nu^\top K^\dagger \nu}$, where $\nu$ captures the maximum within-trajectory variation of process rewards, rather than treating $\tilde{B}$ as a global hyperparameter.
To assess whether this adaptivity is necessary, we sweep fixed values $\tilde{B}\in[0.03,0.1,0.3,1]$, which yield accuracies of $[71\%,78\%,77\%,74\%]$, respectively.
Although in Figure \ref{fig:b_estim} a tuned constant can match the adaptive rule on this small subset, no fixed value is uniformly reliable: the best choice depends on the PRM, base generator, and dataset.
Moreover, Figure~\ref{fig:b_estim} shows that the estimated $\tilde{B}$ varies substantially across reasoning depth and problem instance.
This variation is expected, since the appropriate uncertainty scale is determined by the local geometry of the reward function on the sampled trajectories.
A fixed radius cannot capture this heterogeneity.
Thus, beyond being hyperparameter-free, the adaptive estimate is essential for calibrating the algorithm's conservatism to the local reward geometry.

\paragraph{Choice of distance measure.}
Our RBF kernel computes distances after $\ell_2$-normalizing the PRM representations, making the kernel depend on angular distance rather than raw feature magnitude. 
This follows common practice in language representation spaces, where cosine or angular similarity is widely used for comparing text embeddings and is known to be more robust to anisotropy and norm artifacts in representations~\citep{li2020sentence,su2021whitening}.
When we instead use Euclidean distance in the unnormalized representation space, accuracy drops from $79\%$ to $74\%$.
This result is consistent with prior findings that the direction of language-model representations carries more useful semantic information than their raw norm, confirming that the kernel metric must be chosen to match the geometry of PRM representations.

\begin{figure}[t]
    \centering

    \begin{minipage}[t]{0.485\linewidth}
        \centering
        \includegraphics[width=\linewidth]{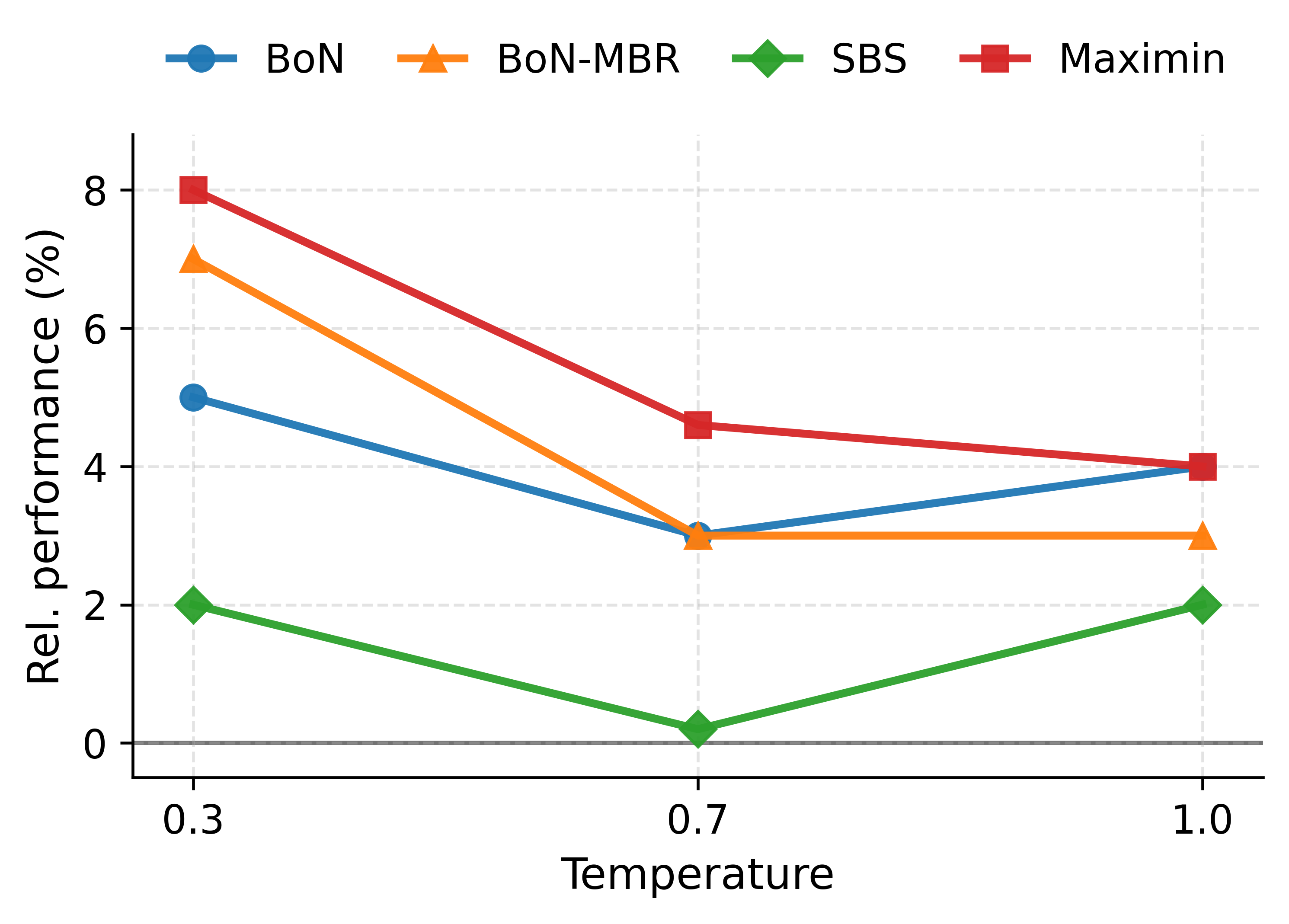}
        \captionof{figure}{\textbf{Temperature sensitivity of decoding strategies.} Curves show the accuracy difference from beam search (in percentage points) across sampling temperatures; positive values indicate improvement over the beam-search baseline.}
        \label{fig:temperature_sensitivity}
    \end{minipage}
    \hfill
    \begin{minipage}[t]{0.485\linewidth}
        \centering       \includegraphics[width=\linewidth]{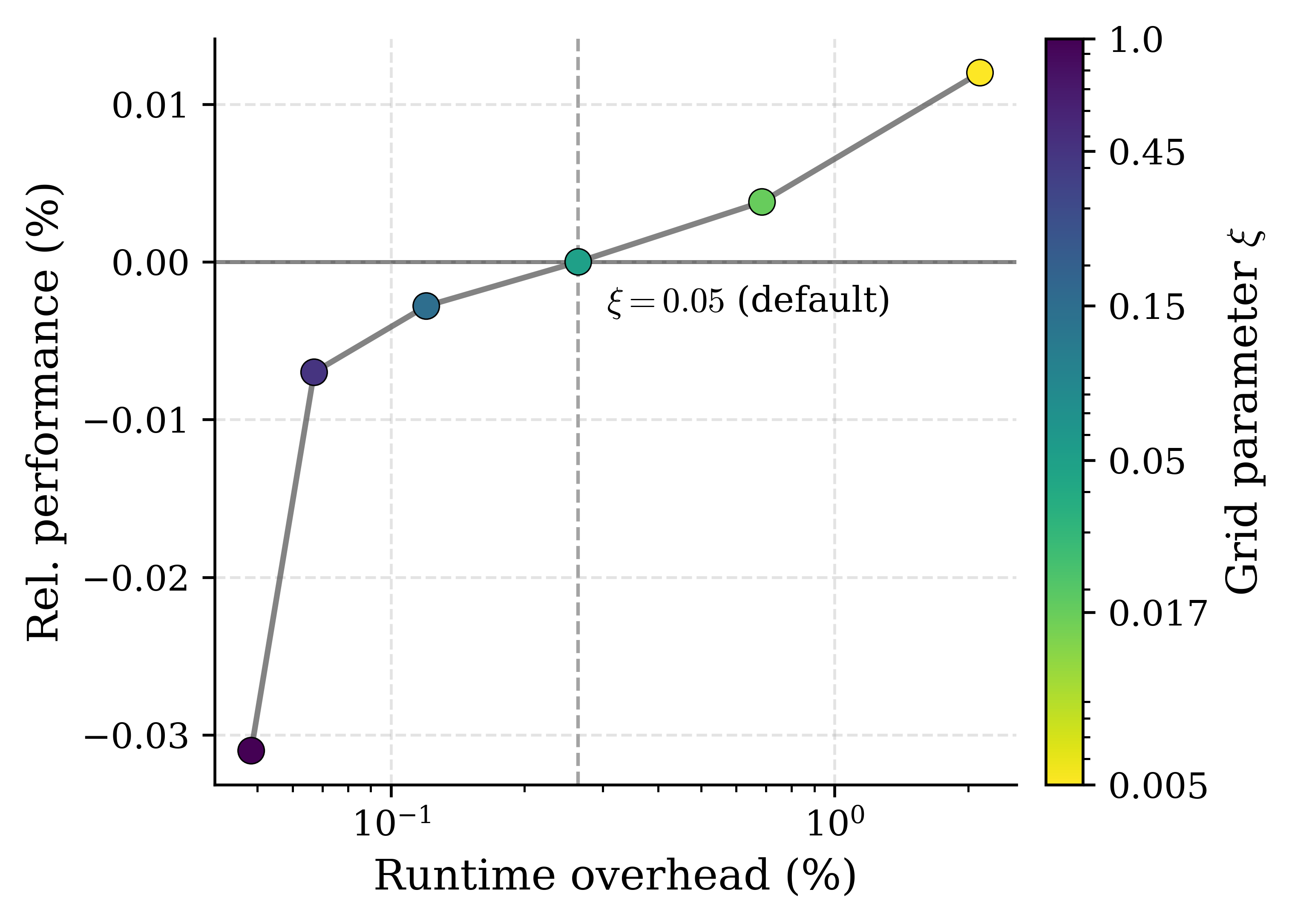}
        \captionof{figure}{\textbf{Sensitivity to the grid parameter $\xi$.} The $x$-axis shows runtime overhead compared to SBS, and the $y$-axis shows the relative accuracy change compared with the default $\xi=0.05$.}
        \label{fig:xi_sensitivity}
    \end{minipage}
\end{figure}

\subsection{Sensitivity analysis}
With the same setup as in \S \ref{subsec:impacts_noise}, we conduct sensitivity analyses with respect to temperature and grid parameters.

\paragraph{Temperature.}
A temperature sensitivity analysis examines how the methods behave as the randomness in generation changes. 
Figure \ref{fig:temperature_sensitivity} shows that maximin is consistently the strongest method across the tested temperatures $[0.3,0.7,1.0]$, improving over beam search by +8.0, +4.6, and +4.0 percentage points, respectively. 
The largest gain appears at the lowest temperature, where sampling is more concentrated and candidate trajectories may be more redundant; in this regime, the kernel penalty can more directly counter concentrated SBS selections. 
At higher temperatures, candidates may already be more diverse, which is consistent with the smaller gap between maximin and SBS.
Overall, maximin remains effective across temperatures.

\paragraph{Grid parameter $\xi$.}
Algorithm \ref{alg:grid-submodular} uses $\xi$ to control the resolution of the multiplicative grid $G_\xi$ over the auxiliary variable $\eta$, with smaller $\xi$ evaluates a finer grid for outer search. 
Figure \ref{fig:xi_sensitivity} illustrates  the relative accuracy under different $\xi$ values change with respect to the default $\xi=0.05$ against the runtime overhead compared to SBS. 
$\xi = 0.05$ incurs 0.26\% runtime overhead.
Relative to the default $\xi=0.05$, finer grids slightly improve accuracy ($+0.38\%$ at $\xi=0.005$; $+1.2\%$ at $\xi=0.017$) while increasing runtime overhead ($2.1\%$ and $0.69\%$). 
Coarser grids reduce overhead, with small changes up to $\xi=0.45$ ($-0.28\%$ at $\xi=0.15$; $-0.7\%$ at $\xi=0.45$), but the coarsest setting $\xi=1.0$ yields a larger drop ($-3.10\%$). 
Thus, performance is stable over moderate grid resolutions, while very fine grids mainly increase cost and very coarse grids can hurt accuracy.

\section{Conclusion}
We studied the vulnerability of PRM-guided search to over-optimization of imperfect step-level reward estimates. Our analyses identify an extreme-value failure mode in which noisy PRM scores can cause search to favor non-viable prefixes as reasoning depth increases, thereby pruning promising trajectories and wasting inference-time compute. 
To address this, we formulated process supervision as robust optimization under plausible reward perturbations and derived a training-free robust selection procedure for PRM-guided decoding. 
Empirically, this approach improves guided search without additional verifier training or test-time adaptation. 
These results suggest that scaling inference-time reasoning requires not only stronger learned evaluators, but also search algorithms that remain reliable when those evaluators are uncertain.

\section*{Acknowledgement}
Taejong Joo was supported in part by the Richard L. Francis Fellowship for this work.

\bibliographystyle{apalike}
\bibliography{references}

\newpage

\appendix

\section{Proof of claims} \label{appx:proof}

\subsection{Proof of Theorem \ref{thm:greedy_hacking_cutoff}}
\greedyhacking*

\begin{proof}
The core idea of the proof is to introduce a deterministic cutoff $a_t$ and lower-bounds to the failure event by the intersection of two simpler events: $\left\{ \max_{\tau_\star\in\gC_t^\star} \widehat V(\tau_\star) \le a_t \right\} \cap \left\{ \max_{\tau_-\in\gC_t^-} \widehat V(\tau_-) > a_t \right\}.$
Specifically, from Assumption \ref{ass:noise}, we can set $a_t= \sup_{\tau'_\star \in \gC_t^\star}  \{ V_t(\tau'_\star) + b_t(\tau'_\star) \} +\lambda$ for any fixed $\lambda \in \R$.
Given the best viable verifier score $M_t^\star = \max_{\tau_\star\in\gC_t^\star}\widehat V(\tau_\star),$ greedy selection fails whenever $M_t^\star\le a_t$ and $\max_{\tau_-\in\gC_t^-}\widehat V(\tau_-)> a_t.$
Indeed, on this event, some non-viable prefix has verifier score at least as large as the best viable prefix. Therefore,
\begin{equation}
    \Pr\left( \widehat\tau_{\mathrm{greedy}}^{(t)} \notin \widetilde{\gP}_t^\star \mid \gC_t \right)
    \ge \Pr\left( M_t^\star\le a_t, \max_{\tau_-\in\gC_t^-}\widehat V(\tau_-)> a_t \mid \gC_t \right).
\end{equation}

The viable and non-viable noise variables are independent, so the two events on the right-hand side are independent, and thus
\begin{equation}
\Pr\left( M_t^\star\le a_t,  \max_{\tau_-\in\gC_t^-}\widehat V(\tau_-)> a_t \mid \gC_t \right) 
    = \Pr(M_t^\star\le a_t\mid \gC_t) \Pr\left( \max_{\tau_-\in\gC_t^-}\widehat V(\tau_-)> a_t \mid \gC_t \right).
\end{equation}

First, for every viable prefix $\tau_\star\in\gC_t^\star$, $\widehat V(\tau_\star) = V(\tau_\star)+\beta_t(\tau_\star)+\varepsilon_t(\tau_\star) \le \sup_{\tau'_\star \in \gC_t^\star}  \{ V_t(\tau'_\star) + b_t(\tau'_\star) \} +\varepsilon_t(\tau_\star).$ 
Thus $M_t^\star\le a_t$ whenever $\varepsilon_t(\tau_\star)\le \lambda$ for all $\tau_\star\in\gC_t^\star$. Therefore, we get 
\begin{equation}
\Pr(M_t^\star\le a_t\mid\gC_t) \ge
\Pr\left( \varepsilon_t(\tau_\star)\le \lambda, \;\forall \tau_\star\in\mathcal C_t^\star \middle|  \mathcal C_t \right) = \Phi\left(\frac{\lambda}{\sigma_t}\right)^{N_t^\star}.
\end{equation} 

Second, for every non-viable prefix $\tau_-\in\gC_t^-$, $\widehat V(\tau_-) = V(\tau_-)+\beta_t(\tau_-)+\varepsilon_t(\tau_-) 
\ge \inf_{\tau'_- \in \gC_t^-} \{V(\tau'_-)+\beta_t(\tau'_-) \} + \varepsilon_t(\tau_-).$
Hence $\widehat V(\tau_-)> a_t$ whenever $\varepsilon_t(\tau_-) \ge a_t-\inf_{\tau'_- \in \gC_t^-} \{V(\tau'_-)+\beta_t(\tau'_-) \}.$
Since $a_t-\inf_{\tau'_- \in \gC_t^-} \{V(\tau'_-)+\beta_t(\tau'_-) \} = \Gamma_t(\gC_t) + \lambda,$ we have 
\begin{equation}
    \Pr(\widehat V(\tau_-)> a_t\mid\gC_t)
    \ge 1-\Phi\left( \frac{\Gamma_t(\gC_t)+\lambda}{\sigma_t} \right)
    = r_t(\lambda).
\end{equation}

Since the non-viable noises are independent by Assumption \ref{ass:noise},
\begin{equation} 
    \Pr\left( \max_{\tau_-\in\gC_t^-}\widehat V(\tau_-)> a_t \mid\gC_t \right) \ge 1-\left(1-r_t(\lambda)\right)^{N_t^-}.
\end{equation}

Combining the two bounds gives
\begin{equation}
    \Pr\left( \widehat\tau_{\mathrm{greedy}}^{(t)} \notin \widetilde{\gP}_t^\star \mid \gC_t \right)
    \ge \Phi\left(\frac{\lambda}{\sigma_t}\right)^{N_t^\star} \left[ 1- \Phi\left( \frac{\Gamma_t(\gC_t)+\lambda}{\sigma_t} \right)^{N_t^-} \right].
\end{equation}

Since this holds for every $\lambda\in\mathbb R$, taking the supremum over $\lambda$ gives the final claim.
\end{proof}

\subsection{Proof of Theorem \ref{thm:generator_averaged_greedy_failure}}
\generatoravg*

\begin{proof}
Let $K_t=N_t^\star=|\mathcal C_t^\star|.$
By the Markovian independent-generation assumption, $K_t\sim \mathrm{Binomial}(N_t,p_t(\mathcal H_t)).$

If $K_t=0$, then no viable continuation is present in the candidate set.
Therefore greedy search necessarily selects a non-viable prefix, and $\Pr\left( \widehat \tau_{\mathrm{greedy}}^{(t)} \notin \widetilde{\mathcal P}_t^\star \;\middle|\; K_t=0,\mathcal H_t \right) = 1.$

We now fix $1\le k\le N_t$. Conditional on $K_t=k$, we have $N_t^\star=k$ and $N_t^-=N_t-k.$
If $k=N_t$, then there are no non-viable candidates, and the verifier-selection failure probability is zero. This case is also handled by \eqref{eqn:temp_formula} because $1-b_t(\lambda;\mathcal H_t)^0=0$.

For $1\le k\le N_t-1$, Theorem~\ref{thm:greedy_hacking_cutoff} gives, for any realized mixed candidate set $\mathcal C_t$, 
\begin{equation} \label{eqn:temp_formula}
    \Pr\left( \widehat \tau_{\mathrm{greedy}}^{(t)} \notin \widetilde{\mathcal P}_t^\star \;\middle|\; \mathcal C_t \right) \ge \Phi\left(\frac{\lambda}{\sigma_t}\right)^k \left[
1- \Phi\left( \frac{\Gamma_t(\mathcal C_t)+\lambda}{\sigma_t} \right)^{N_t-k} \right].
\end{equation}

Using the envelope assumption $\Gamma_t(\mathcal C_t)\le \overline \Gamma_t(\mathcal H_t),$ and the monotonicity of $\Phi$, we obtain $\Phi\left( \frac{\Gamma_t(\mathcal C_t)+\lambda}{\sigma_t} \right) \le b_t(\lambda;\mathcal H_t).$
Therefore,
\begin{equation}
    \Pr\left( \widehat \tau_{\mathrm{greedy}}^{(t)} \notin \widetilde{\mathcal P}_t^\star \;\middle|\; K_t=k,\mathcal H_t \right) \ge a_t(\lambda)^k \left[ 1-b_t(\lambda;\mathcal H_t)^{N_t-k} \right].
\end{equation}

Averaging over $K_t\sim \mathrm{Binomial}(N_t,p_t(\mathcal H_t))$  gives
\begin{multline}
    \Pr\left( \widehat \tau_{\mathrm{greedy}}^{(t)} \notin \widetilde{\mathcal P}_t^\star \;\middle|\; \mathcal H_t \right) \ge \Pr(K_t=0\mid\mathcal H_t) + \sum_{k=1}^{N_t} \Pr(K_t=k\mid\mathcal H_t) a_t(\lambda)^k \left[ 1-b_t(\lambda;\mathcal H_t)^{N_t-k} \right]  \\
    = \left(1-p_t(\mathcal H_t)\right)^{N_t} + \sum_{k=1}^{N_t} {N_t\choose k} p_t(\mathcal H_t)^k \left(1-p_t(\mathcal H_t)\right)^{N_t-k} a_t(\lambda)^k \left[ 1-b_t(\lambda;\mathcal H_t)^{N_t-k} \right].
\end{multline}

Since the bound holds for every $\lambda\in\mathbb R$, taking the supremum over
$\lambda$ gives the final claim.
\end{proof}

\subsection{Proof of Corollary \ref{cor:depth_accumulation_generator}}
\depthaccumulation*

\begin{proof}
By Theorem~\ref{thm:generator_averaged_greedy_failure}, on any surviving history $\mathcal H_t$, $\Pr(\mathcal S_t^c\mid \mathcal H_t,\mathcal S_{t-1}) \ge \delta_t^\star(\mathcal H_t).$ Equivalently, $\Pr(\mathcal S_t\mid \mathcal H_t,\mathcal S_{t-1}) \le 1-\delta_t^\star(\mathcal H_t).$

Using the tower property recursively yields
\begin{equation}
    \Pr(\mathcal S_{T_{\max}}) \le \mathbb E\left[ \prod_{t=1}^{T_{\max}} \left(1-\delta_t^\star(\mathcal H_t)\right) \right].
\end{equation}

Taking complements gives the claimed failure lower bound.
\end{proof}

\subsection{Proof of Theorem \ref{thm:rkhs_worst_case}}
\rkhsworstcase*
\begin{proof}
By the reproducing property, $e(\tau_i)=\langle e, k(\tau_i,\cdot)\rangle_{\Phi_k}.$
Thus, we have 
\begin{equation}
    \sum_{i=1}^N q_i e(\tau_i) = \left\langle e, \sum_{i=1}^N q_i k(\tau_i,\cdot)\right\rangle_{\Phi_k}.
\end{equation}

By applying Cauchy--Schwarz, we get
\begin{equation} \label{tmp:cs_}
    \left| \sum_{i=1}^N q_i e(\tau_i) \right| \le \|e\|_{\Phi_k} \left\| \sum_{i=1}^N q_i k(\tau_i,\cdot) \right\|_{\Phi_k}.
\end{equation}

In \eqref{tmp:cs_}, the second factor is exactly
\begin{equation}
    \left\| \sum_{i=1}^N q_i k(\tau_i,\cdot) \right\|_{\Phi_k}^2 = \sum_{i,j=1}^N q_i q_j k(\tau_i,\tau_j) = q^\top K q.
\end{equation}

Therefore, we have
\begin{equation}
    \left| \sum_{i=1}^N q_i e(\tau_i) \right| \le B\sqrt{q^\top K q}.
\end{equation}

Equality is achieved by choosing \(e\) proportional to \(\sum_i q_i k(\tau_i,\cdot)\), which proves \eqref{eq:rkhs_bound}. 

\end{proof}

\subsection{Proof of Theorem \ref{thm:end-to-end}}
\endtoendanalysis*
\begin{proof}
Let \(q^\star\) be optimal for the original objective \(J\), and let $\eta^\star \triangleq \sqrt{{q^\star}^\top K q^\star} \in [\sqrt m,m].$
Since \(\gG_\xi\) is a multiplicative grid on this interval, there exists \(t^\star\) such that $\eta^\star \le \eta_{t^\star} \le (1 + \xi)\eta^\star.$
Using the variational representation in~\eqref{eq:J-eta-exact}, we get $J(q^\star)=J_{\eta^\star}(q^\star)$ and $J_{\eta_{t^\star}}(q^\star) = s^\top q^\star-\frac{\tilde{B}}{2\eta_{t^\star}}{q^\star}^\top K q^\star-\frac{\tilde{B}}{2}\eta_{t^\star}.$
Writing \(z^\star\triangleq{q^\star}^\top K q^\star\) and using \(\eta^\star=\sqrt{z^\star}\), we obtain
\begin{equation}
    J(q^\star)-J_{\eta_{t^\star}}(q^\star) = \tilde{B} \left[ \frac{z^\star}{2\eta_{t^\star}}+\frac{\eta_{t^\star}}{2}-\sqrt{z^\star} \right].
\end{equation}

Let \(r \triangleq \eta_{t^\star}/\eta^\star \in [1,1 + \xi]\). Then, we get 
\begin{equation}
    \frac{z^\star}{2\eta_{t^\star}}+\frac{\eta_{t^\star}}{2}-\sqrt{z^\star} = \frac{\eta^\star}{2} \left(\frac1r+r-2\right) = \frac{\eta^\star}{2}\cdot \frac{(r-1)^2}{r}
    \le \frac{\eta^\star \xi^2}{2(1 + \xi)}
    \le \frac{m\xi^2}{2(1 + \xi)},
\end{equation}
where the last step uses \(\eta^\star \le m\). Hence
\begin{equation}
\label{eq:grid-error}
J_{\eta_{t^\star}}(q^\star) \ge J(q^\star)-\frac{\tilde{B} m\xi^2}{2(1 + \xi)}.
\end{equation}

Let \(q^\dagger_{\eta_{t^\star}}\) denote the candidate returned by the inner solver at grid point \(\eta_{t^\star}\). By the randomized greedy guarantee and the objective-improving post-processing step, we have
\begin{equation}
    \mathbb{E} \left[J_{\eta_{t^\star}} \left(q^\dagger_{\eta_{t^\star}}\right)\right] \ge \mathbb{E} \left[J_{\eta_{t^\star}} \left(q^{\rm rg}_{\eta_{t^\star}}\right)\right] \ge \rho \max_{q\in Q} J_{\eta_{t^\star}}(q).
\end{equation}

Since $q^\star$ is feasible, we have
\begin{equation}
    \mathbb{E} \left[J_{\eta_{t^\star}} \left(q^\dagger_{\eta_{t^\star}}\right)\right] \ge \rho J_{\eta_{t^\star}}(q^\star).    
\end{equation}

On the other hand, Algorithm \ref{alg:grid-submodular} returns $q^\dagger \in \arg\max_{\eta\in G_\xi} J(q^\dagger_\eta).$ Therefore, pointwise, we have
\begin{equation}
    J(q^\dagger) \ge J \left(q^\dagger_{\eta_{t^\star}}\right).
\end{equation}

Moreover, since \(J(q)=\max_{\eta>0}J_\eta(q)\), we have
\begin{equation}
    J \left(q^\dagger_{\eta_{t^\star}}\right) \ge J_{\eta_{t^\star}} \left(q^\dagger_{\eta_{t^\star}}\right).
\end{equation}

Combining the previous inequalities and taking expectations gives
\begin{equation}
    \mathbb{E}[J(q^\dagger)] \ge \rho J_{\eta_{t^\star}}(q^\star).
\end{equation}

Using \eqref{eq:grid-error}, we obtain
\begin{equation}
    \mathbb{E}[J(q^\dagger)] \ge \rho\left( J(q^\star) - \frac{\widetilde B m\xi^2}{2(1+\xi)} \right),
\end{equation}
which proves~\eqref{eq:end-to-end-main}.

\end{proof}

\section{Additional theoretical results}

\subsection{Beam-level reward hacking via deterministic cutoff} \label{appx:subsec:beam_hacking}
\begin{corollary}
\label{thm:beam_hacking_cutoff}
Under Assumption \ref{ass:noise} and any fixed $\lambda\in \R$, let $\Gamma_t(\gC_t) = \sup_{\tau_\star\in\gC_t^\star, \tau_-\in\gC_t^-} \{V(\tau_\star)+\beta_t(\tau_\star) -V(\tau_-)-\beta_t(\tau_-) \}$, $q_t(\lambda)= \Phi\left(\frac{\lambda}{\sigma_t}\right)^{N_t^\star}$, and $r_t(\lambda) = 1-\Phi\left( \frac{\Gamma_t (\gC_t)+\lambda}{\sigma_t} \right)$.
When $N_t^\star\ge 1$ and $N_t^- \ge w_t$, the probability that the beam of width $w_t$ contains no viable prefix satisfies
\begin{equation} \label{eq:beam_search_lowerbound}
\Pr\left( \mathcal B_t\cap \widetilde{\gP}_t^\star=\emptyset  \middle|  \gC_t \right)  
    \ge q_t(\lambda)   \Pr\left( \mathrm{Binomial}(N_t^-,r_t)\ge w_t \right),    
\end{equation}
where $\Pr\left( \mathrm{Binomial}(N_t^-,r_t)\ge w_t \right) \triangleq \sum_{j=w_t}^{N_t^-}{N_t^-\choose j} r_t(\lambda)^j \left(1- r_t(\lambda)\right)^{N_t^- - j}$.
\end{corollary} 
\begin{proof}
Let $M_t^\star = \max_{\tau_\star\in\gC_t^\star} \widehat V(\tau_\star)$ be the largest verifier score among viable prefixes. 
Also, for any fixed $\lambda \in \R$, set $a_t= \sup_{\tau'_\star \in \gC_t^\star}  \{ V_t(\tau'_\star) + b_t(\tau'_\star) \} +\lambda.$
Beam search fails at step $t$ if the top $w_t$ prefixes selected by $\widehat V$ are all non-viable. A sufficient condition for this failure is $M_t^\star\le a_t$ and $\left| \left\{ \tau_-\in\gC_t^-: \widehat V(\tau_-) > a_t \right\} \right| \ge w_t.$
Indeed, on this event, at least $w_t$ non-viable prefixes have verifier score at least as large as every viable prefix. Thus the top-$w_t$ beam can be filled entirely by non-viable prefixes.

By the proof of Theorem \ref{thm:greedy_hacking_cutoff}, we have 
\begin{equation}
\Pr(M_t^\star\le a_t\mid\gC_t)
    \ge \Phi\left(\frac{\lambda}{\sigma_t}\right)^{N_t^\star}
    = q_t(\lambda)
\end{equation}
and
\begin{equation}
\Pr(\widehat V(\tau_-) > a_t\mid\gC_t)
    \ge 1 - \Phi\left(\frac{\Gamma_t (\gC_t)+\lambda}{\sigma_t} \right)
= r_t(\lambda).
\end{equation}

Since the non-viable noises are independent by Assumption \ref{ass:noise}, the number of non-viable prefixes exceeding the cutoff stochastically dominates $\mathrm{Binomial}\left(N_t^-,r_t(\lambda)\right).$
Consequently, we get
\begin{equation}
    \Pr\left( \left| \left\{ \tau_-\in\gC_t^-: \widehat V(\tau_-)> a_t \right\} \right| \ge w_t  \middle|  \gC_t \right) \ge \Pr\left( \mathrm{Binomial}(N_t^-,r_t)\ge w_t \right).
\end{equation}

As in Theorem \ref{thm:greedy_hacking_cutoff}, the event $\{M_t^\star\le a_t\}$ and the event that at least $w_t$ non-viable prefixes exceed $a_t$ depends only on the non-viable noises being independent. Hence, we get our desired result as
\begin{align}
&\Pr\left( \mathcal B_t\cap\widetilde{\gP}_t^\star=\emptyset  \middle|  \gC_t \right) \\
    &\qquad\ge \Pr\left( M_t^\star\le a_t,  \left| \left\{ \tau_-\in\gC_t^-: \widehat V(\tau_-)> a_t \right\} \right| \ge w_t  \middle|  \gC_t \right) \\
    &\qquad= \Pr(M_t^\star\le a_t\mid\gC_t) \Pr\left( \left| \left\{ \tau_-\in\gC_t^-: \widehat V(\tau_-)> a_t \right\} \right| \ge w_t  \middle|  \gC_t \right) \\
    &\qquad\ge q_t(\lambda)   \Pr\left( \mathrm{Binomial}(N_t^-,r_t)\ge w_t \right).
\end{align}

\end{proof}

\subsection{Kernel penalty breaks narrow beam-search failures}
\label{appx:kernel_narrow_failure}

\begin{theorem}
\label{thm:clustered_failure_rescue}
Let us fix a depth \(t\) with \(C_t^\star\neq\emptyset\). Let
\(\widehat B_t^{\rm prm}\) be the top-\(w_t\) PRM beam, with \(w_t\ge 2\), and suppose it fails (i.e., $\widehat B_t^{\rm prm}\subseteq C_t^-$). 
For \(|S_t|=w_t\), we let
\begin{equation}
    J_t(S_t) = \sum_{\tau_i^{(t)}\in S_t}s_{t,i} - \widetilde B_t \sqrt{q_t(S_t)^\top K q_t(S_t)},
\end{equation}
where \(\widetilde B_t>0\), \(s_{t,i}=\widehat V(\tau_i^{(t)})\), and \(q_t(S_t)\in\{0,1\}^{|C_t|}\) is the indicator vector of \(S_t\). Let us assume \(K\) is symmetric, \(K_{ii}=1\), and \(0\le K_{ij}\le 1\).

Let
\begin{equation}
    \tau_{a_t}^{(t)} \in \arg\min_{\tau_i^{(t)}\in \widehat B_t^{\rm prm}} s_{t,i}, \qquad
        \tau_{v_t}^{(t)} \in \arg\max_{\tau_i^{(t)}\in C_t^\star} s_{t,i},
\end{equation}
and define the PRM edge gap
\begin{equation}
    \Delta_t^{\rm edge}=s_{t,a_t}-s_{t,v_t}.
\end{equation}

Suppose the failed PRM beam is more similar to its weakest selected prefix than to the best viable replacement; that is, $K_{a_t i}\ge \alpha_t$ and $K_{v_t i}\le \beta_t$ for all $\tau_i^{(t)}\in \widehat B_t^{\rm prm}\setminus\{\tau_{a_t}^{(t)}\}$ given \(\alpha_t>\beta_t\).
In this case, if 
\begin{equation}
    0\le \Delta_t^{\rm edge} < \widetilde B_t \frac{w_t-1}{w_t} (\alpha_t-\beta_t),
\end{equation}
then \(\widehat B_t^{\rm prm}\) is not a maximizer of \(J_t\).
\end{theorem}

\begin{proof}
Since \(\widehat B_t^{\rm prm}\) is the top-\(w_t\) PRM beam and
\(\tau_{v_t}^{(t)}\notin \widehat B_t^{\rm prm}\), the weakest selected score
is at least \(s_{t,v_t}\). Hence we have $\Delta_t^{\rm edge}=s_{t,a_t}-s_{t,v_t}\ge 0 .$

Let
\begin{equation}
    S_t=\widehat B_t^{\rm prm},
    \qquad
    S_t'
    =
    \left(
    \widehat B_t^{\rm prm}
    \setminus\{\tau_{a_t}^{(t)}\}
    \right)
    \cup
    \{\tau_{v_t}^{(t)}\},
\end{equation}
and we denote
\begin{equation}
    R_t(S_t)=\sqrt{q_t(S_t)^\top Kq_t(S_t)} .
\end{equation}
Then, we get
\begin{equation}
    J_t(S_t')-J_t(S_t)
    =
    -\Delta_t^{\rm edge}
    +
    \widetilde B_t
    \left(R_t(S_t)-R_t(S_t')\right).
\end{equation}

It remains to lower-bound the reduction in kernel penalty. Let us define
\begin{equation}
    G_t
    =
    \sum_{\tau_i^{(t)}\in S_t\setminus\{\tau_{a_t}^{(t)}\}}
    \left(K_{a_t i}-K_{v_t i}\right).
\end{equation}
By the assumption about \(\alpha_t,\beta_t\), we have
\begin{equation}
    G_t
    \ge
    (w_t-1)(\alpha_t-\beta_t).
\end{equation}
The swap only changes kernel terms involving \(\tau_{a_t}^{(t)}\) or
\(\tau_{v_t}^{(t)}\). Since \(K\) is symmetric and
\(K_{a_ta_t}=K_{v_tv_t}=1\),
\begin{equation}
    R_t(S_t)^2-R_t(S_t')^2=2G_t .
\end{equation}
Also, since \(0\le K_{ij}\le 1\), every set \(T_t\) of size \(w_t\)
satisfies \(R_t(T_t)\le w_t\). Therefore, we have 
\begin{equation}
    R_t(S_t)-R_t(S_t') = \frac{R_t(S_t)^2-R_t(S_t')^2} {R_t(S_t)+R_t(S_t')} = \frac{2G_t}{R_t(S_t)+R_t(S_t')} \ge \frac{G_t}{w_t} \ge \frac{w_t-1}{w_t}(\alpha_t-\beta_t).
\end{equation}
Thus, it holds that 
\begin{equation}
    J_t(S_t')-J_t(S_t) \ge -\Delta_t^{\rm edge} + \widetilde B_t \frac{w_t-1}{w_t} (\alpha_t-\beta_t).
\end{equation}
Therefore, if 
\begin{equation}
    \Delta_t^{\rm edge}
    <
    \widetilde B_t
    \frac{w_t-1}{w_t}
    (\alpha_t-\beta_t),
\end{equation}
then \(J_t(S_t')>J_t(S_t)\). Hence
\(\widehat B_t^{\rm prm}\) cannot be a maximizer of \(J_t\).
\end{proof}

\subsection{Shifted nonnegative surrogates} \label{appx_subsec:corollary}
\begin{corollary}[End-to-end guarantee with shifted nonnegative surrogates]
\label{cor:shifted-end-to-end}
Assume the same conditions as in Theorem \ref{thm:end-to-end}, without the non-negativity of $J_\eta(q)$. 
For each grid point \(\eta \in G_\xi\), let $Q_{\le m}\triangleq \{q\in\{0,1\}^N:\mathbf 1^\top q\le m\}$ and $L_\eta \triangleq \min_{q\in Q_{\le m}} J_\eta(q).$

Let \(q^\dagger\) be the final output selected by Algorithm~\ref{alg:grid-submodular} according to the original objective \(J\). 
Let $q^\star\in\arg\max_{q\in\mathcal Q}J(q)$ and define $L_{\mathcal H} \triangleq \min_{\eta\in\gG_\xi}L_\eta.$  
Then, it holds that 
\[
    \mathbb E[J(q^\dagger)]
    \ge
    \rho
    \left(
        J(q^\star)
        -
        \frac{\tilde B m\xi^2}{2(1+\xi)}
    \right)
    +
    (1-\rho)L_{\mathcal H}.
\]
\end{corollary}

\begin{proof}
For any fixed grid point \(\eta\), the shifted objective $\bar J_\eta(q)=J_\eta(q)-L_\eta$ is nonnegative on \(Q_{\le m}\), and hence on every intermediate set considered by the randomized greedy procedure.

By the assumed
multiplicative guarantee in Theorem \ref{thm:end-to-end}, we have
\begin{equation}
    \mathbb E[\bar J_\eta(q_\eta^{\mathrm{rg}})] \ge \rho \max_{q\in\mathcal Q}\bar J_\eta(q) = \rho\left(\max_{q\in\mathcal Q}J_\eta(q)-L_\eta\right).
\end{equation}

Since \(\bar J_\eta(q)=J_\eta(q)-L_\eta\), this implies $\mathbb E[J_\eta(q_\eta^{\mathrm{rg}})] \ge \rho \max_{q\in\mathcal Q}J_\eta(q) + (1-\rho)L_\eta.$
The local-search post-processing only improves \(J_\eta\), so
\begin{equation}
    \mathbb E[J_\eta(q_\eta^\dagger)] \ge \rho \max_{q\in\mathcal Q}J_\eta(q) + (1-\rho)L_\eta.
\end{equation}

Let \(q^\star\in\arg\max_{q\in\mathcal Q}J(q)\) and $\alpha^\star = \sqrt{{q^\star}^\top Kq^\star}$, and then consider \(\eta^\star_{\xi}\in\gG_\xi\) such that
\begin{equation}
    \alpha^\star \le \eta^\star_{\xi} \le (1+\xi)\alpha^\star.
\end{equation}
Then, by using the same technique for bounding the multiplicative grid as in Theorem \ref{thm:end-to-end}, we have
\begin{equation}
    J(q^\star)-J_{\eta^\star_{\xi}}(q^\star)
    =
    \tilde B
    \left(
        \frac{(\eta^\star_{\xi}-\alpha^\star)^2}{2\eta^\star_{\xi}}
    \right)
    \le
    \frac{\tilde B \xi^2}{2(1+\xi)}\alpha^\star
    \le
    \frac{\tilde B m\xi^2}{2(1+\xi)}.
\end{equation}

Therefore,
\begin{equation}
    \max_{q\in\mathcal Q}J_{\eta^\star_{\xi}}(q) \ge J_{\eta^\star_{\xi}}(q^\star) \ge J(q^\star) - \frac{\tilde B m\xi^2}{2(1+\xi)}.
\end{equation}

Applying the shifted guarantee at \(\eta^\star_{\xi}\), and using
\(L_{\eta^\star_{\xi}}\ge L_{\mathcal H}\), gives
\begin{equation}
    \mathbb E[J_{\eta^\star_{\xi}}(q_{\eta^\star_{\xi}}^\dagger)] \ge \rho \left( J(q^\star) - \frac{\tilde B m\xi^2}{2(1+\xi)} \right) + (1-\rho)L_{\mathcal H}.
\end{equation}
Finally, Algorithm~\ref{alg:grid-submodular} returns the candidate with the largest original
objective \(J\) among all grid solutions, and \(J(q)\ge J_\eta(q)\) for every \(q\) and \(\eta\).
Thus,
\begin{equation}
    J(q^\dagger) \ge J(q_{\eta^\star_{\xi}}^\dagger) \ge J_{\eta^\star_{\xi}} (q_{\eta^\star_{\xi}}^\dagger),
\end{equation}
and the claimed bound follows after taking expectations.
\end{proof}

Corollary~\ref{cor:shifted-end-to-end} shows why the nonnegativity shift cannot be treated as
without loss for a multiplicative approximation guarantee. Shifting preserves the optimizer and
the marginal gains, but it changes the scale on which the multiplicative approximation is measured.
The resulting guarantee therefore contains the additional additive term
\((1-\rho)L_{\mathcal H}\), which disappears only when the unshifted surrogates are already
nonnegative.
The shift is used only for the approximation analysis. The implementation does not require
computing \(L_\eta\), since subtracting a constant from \(J_\eta\) does not change marginal gains,
local swaps, or the selected set.

\section{Comparison with MBR and diversity-based decoding} \label{app:diversity_comparison}

Maximin search belongs to a broader family of inference-time methods that use the geometry of a candidate set to avoid depending too strongly to a single imperfect score. 
The key distinction is what decision geometry is used to regularize. 
Minimum Bayes Risk (MBR) uses geometry to choose a representative final output. Diversity-promoting methods use geometry to reduce redundancy among candidates. Maximin search uses geometry to make intermediate pruning robust to verifier misspecification.

\paragraph{Maximin vs MBR}
We first contrast maximin search with MBR decoding. Given completed candidates $\gY=\{y_i\}_{i=1}^N$, sample-based MBR selects
\begin{equation}
    \hat y_{\mathrm{MBR}}
    \in
    \arg\max_{y\in\gY}
    \frac{1}{N}\sum_{i=1}^N U(y,y_i).
\end{equation}
When the utility is the negative of a metric distance, this is equivalent to minimizing empirical expected distance, so MBR selects a Fr\'echet median of the empirical candidate distribution
$\hat p_N=\frac{1}{N}\sum_{i=1}^N\delta_{y_i}$, i.e., an empirical $1$-median \citep{jinnai2025regularized}. This view also underlies MBR-regularized Best-of-N decoding, where an MBR term discourages selecting high-reward outputs that are far from the reference-policy sample distribution \citep{jinnai2025regularized}.

Thus, MBR and maximin search share a common motivation: both use candidate-set geometry to reduce the risk of over-optimizing a noisy or misspecified reward. However, they use geometry in different ways. MBR is a consensus rule over completed outputs. It favors candidates that are central under the empirical output distribution. This is desirable when centrality correlates with correctness, because isolated high-reward outliers are suppressed. In mathematical reasoning, however, centrality can also be misleading: many sampled chains may share the same plausible but incorrect misconception, while a correct continuation may be rare. A purely central candidate can therefore correspond to a dominant wrong reasoning mode.

Maximin search addresses a different problem. It does not choose a single central completed output. Instead, it selects a subset of partial reasoning prefixes to preserve for future expansion. The goal is to avoid concentrating the search frontier in one high-similarity region when the verifier may be wrong in a coherent way on that region. In this sense, maximin search regularizes the pruning step rather than the final selection step. It preserves optionality under verifier uncertainty, whereas MBR aggregates evidence to choose a representative final candidate.

\paragraph{Maximin vs diversity-based decoding}
Maximin search is also distinct from generic diversity regularization. Diverse beam search introduces diversity as a search heuristic, often at the lexical or semantic level, to reduce redundancy among beams \citep{vijayakumar2016diverse}. Determinantal point process (DPP) methods similarly select high-quality diverse subsets by optimizing a determinantal objective, typically rewarding large volume or near-orthogonality in the feature space \citep{kulesza2012determinantal,meister2021determinantal}. These methods are natural baselines because they also discourage redundant candidates. However, their diversity terms are not derived from a model of verifier error.

The maximin penalty has a different interpretation. Under our RKHS misspecification model, the term $\sqrt{q^\top K q}$ is proportional to the worst-case coherent verifier error over the selected trajectories. The regularization coefficient $\lambda$ therefore plays the role of an uncertainty radius, not merely a diversity weight. This interpretation yields a principled way to choose $\lambda$ from an estimated level of verifier misspecification, rather than treating it as an arbitrary diversity hyperparameter. 
The resulting subset-selection problem is
\begin{equation}
    \max_{q\in\{0,1\}^N,\ \mathbf 1^\top q=m}
    s^\top q-\lambda\sqrt{q^\top Kq},
\end{equation}
where $s$ contains the verifier scores and $K$ measures similarity between prefixes. The penalty controls the RKHS norm of the selected empirical measure, which is exactly the quantity that bounds worst-case aggregate misspecification error on the retained trajectories.

Hence the selected set under maximin search is not merely diverse or representative, unlike MBR, diverse beam search, or DDP; it is the highest-scoring trajectories whose aggregate score remains stable under structured step-level reward misspecification.

\section{Experimental details} \label{appx:sec:details}

Given a set of trajectories \(\{\tau_i\}_{i=1}^n\), we use the last hidden state of the PRM as an embedding \(e_i=\mathrm{emb}(\tau_i)\in\mathbb{R}^d\). 
Then, the kernel is defined as $d^2(\tau_i,\tau_j) = 2-2\frac{e_i^\top e_j}{\|e_i\|_2\|e_j\|_2}$, with the bandwidth determined by the median pairwise distance excluding diagonal self-distances:
\begin{equation}
    \gamma = 1/(2\sigma^2), \quad \sigma=\sqrt{\mathrm{median}\{d^2(\tau_i,\tau_j):i\neq j\}}.
\end{equation}
We clamp the median distance and \(\sigma^2\) by a constant \(10^{-6}\) for numerical stability. 


For both generators, we use top-$p$ sampling with $p=0.9$ and temperature $0.7$.
Also, for a step-level generation, we use a stop token of ``\texttt{\textbackslash n\textbackslash n}'' with the maximum number of tokens in each step being 128. 
We set the maximum reasoning depth as 30. To ensure a fair comparison across methods, we halt any trajectory once its total generated length exceeds 2,048 tokens.
Finally, for PRM-guided search methods, we use a constant branching factor of 4 at every reasoning depth.
Thus, under $4m$ number of maximum trajectories, each search state maintains $m$ parent trajectories and expands each of them with four continuations.

We use zero-shot chain-of-thought prompts for the generators, with an output format instruction. Exact prompts are given below.

\begin{prompt}[Prompt for Qwen2.5-Math-1.5B]
\begin{small}    
\begin{Verbatim}[breaklines=True]
Below is an instruction that describes a task. Write a response that appropriately completes the request.

### Instruction:
{question}

### Response: Please reason step by step, and put your final answer within \\boxed{}.

\end{Verbatim}
\end{small}
\end{prompt}

\;

\begin{prompt}[Prompt for Phi-3.5-mini-instruct]
\begin{small}    
\begin{Verbatim}[breaklines=True]
<|system|>You are a helpful assistant<|end|>
<|user|>Below is an instruction that describes a task. Write a response that appropriately completes the request.

### Instruction:
{question}

### Response: Please reason step by step, and put your final answer within \\boxed{}. 

<|end|>
<|assistant|>
\end{Verbatim}
\end{small}
\end{prompt}

We used the following constants for Algorithm~\ref{alg:grid-submodular}. 
The outer multiplicative grid is set based on \(\xi=0.05\), with endpoints \(\eta_{\min}=\sqrt{m}\) and \(\eta_{\max}=m\). 
For each fixed \(\eta\), we initialized the inner solver using the randomized greedy warm start as described in \S \ref{sec:method}, together with deterministic greedy and random initializations. 
Each initialization was refined by \(1\)-swap local search with at most \(30\) improvement passes and improvement tolerance \(10^{-12}\). 
Among the resulting local optima for a fixed \(\eta\), we kept the one with the largest fixed-\(\eta\) surrogate value.

In Table \ref{tab:qwen25_phi_results}, we list the average accuracy for AIME datasets. Next, we give the standard deviations of the AIME results in Table \ref{tab:aime_std_results}.

\begin{table}[h!]
\centering
\caption{
\textbf{
Standard deviations of benchmark results on AIME datasets.}
We report the standard deviations corresponding to the AIME results in Table~\ref{tab:qwen25_phi_results}.}
\label{tab:aime_std_results}
\small
\setlength{\tabcolsep}{4pt}
\begin{tabular}{llcccc}
\toprule
\multirow{2}{*}{Generator} & \multirow{2}{*}{Method}
& \multicolumn{2}{c}{AIME'24}
& \multicolumn{2}{c}{AIME'25} \\
\cmidrule(lr){3-4}
\cmidrule(lr){5-6}
& & $N=16$ & $N=64$ & $N=16$ & $N=64$ \\
\midrule
\multirow{4}{*}{Qwen2.5-Math-1.5B}
& BoN
& 3.34 & 3.34
& 3.85 & 5.09 \\

& BoN-MBR
& 1.93 & 3.34
& 3.85 & 3.85 \\

& SBS
& 1.92 & 5.09
& 3.34 & 3.85 \\

& Maximin
& 3.34 & 3.34
& 3.85 & 1.92 \\

\midrule
\multirow{4}{*}{Phi-3.5-mini-instruct}
& BoN
& 3.34 & 1.93
& 1.92 & 3.34 \\

& BoN-MBR
& 3.34 & 3.34
& 1.92 & 3.34 \\

& SBS
& 1.93 & 1.92
& 1.92 & 3.34 \\

& Maximin
& 1.92 & 3.85
& 1.92 & 1.92 \\
\bottomrule
\end{tabular}
\end{table}

\end{document}